%% file: example_paper.tex
\documentclass{article}

\input{math_commands.tex}

\usepackage{microtype}
\usepackage{graphicx}
\usepackage{subcaption}

\usepackage{booktabs} 
\usepackage{multirow}

\usepackage{hyperref}

\usepackage[accepted]{icml2025}

\usepackage{amsmath}
\usepackage{amssymb}
\usepackage{mathtools}
\usepackage{amsthm}

\usepackage[capitalize,noabbrev]{cleveref}

\usepackage{algorithm}
\usepackage{algorithmic}

\usepackage{colortbl}
\definecolor{lg}{gray}{0.9}
\definecolor{dg}{RGB}{0,150,0}
\definecolor{dr}{RGB}{139,0,0}

\theoremstyle{plain}
\newtheorem{theorem}{Theorem}[section]
\newtheorem{proposition}[theorem]{Proposition}
\newtheorem{lemma}[theorem]{Lemma}

\theoremstyle{definition}
\newtheorem{definition}[theorem]{Definition}

\theoremstyle{remark}
\newtheorem{remark}[theorem]{Remark}

\usepackage[textsize=tiny]{todonotes}

\icmltitlerunning{Sample-Specific Noise Injection For Diffusion-Based Adversarial Purification}

\begin{document}

\twocolumn[
\icmltitle{Sample-Specific Noise Injection For Diffusion-Based Adversarial Purification}



\icmlsetsymbol{equal}{*}

\begin{icmlauthorlist}
\icmlauthor{Yuhao Sun}{equal,melb}
\icmlauthor{Jiacheng Zhang}{equal,melb}
\icmlauthor{Zesheng Ye}{equal,melb}
\icmlauthor{Chaowei Xiao}{wisconsin}
\icmlauthor{Feng Liu}{melb}
\end{icmlauthorlist}

\icmlaffiliation{melb}{School of Computing and Information Systems, The University of Melbourne}
\icmlaffiliation{wisconsin}{University of Wisconsin, Madison}

\icmlcorrespondingauthor{Feng Liu}{fengliu.ml@gmail.com}

\icmlkeywords{Machine Learning, ICML}

\vskip 0.3in
]



\printAffiliationsAndNotice{\icmlEqualContribution} 

\begin{abstract}
\emph{Diffusion-based purification} (DBP) methods aim to remove adversarial noise from the input sample by first injecting Gaussian noise through a forward diffusion process, and then recovering the clean example through a reverse generative process. 
In the above process, how much Gaussian noise is injected to the input sample is key to the success of DBP methods, which is controlled by a constant noise level $t^*$ for all samples in existing methods.
In this paper, we discover that an optimal $t^*$ for each sample indeed could be different.
Intuitively, the cleaner a sample is, the less the noise it should be injected, and vice versa.
Motivated by this finding, we propose a new framework, called \emph{\textbf{S}ample-specific \textbf{S}core-aware \textbf{N}oise \textbf{I}njection} (SSNI). 
Specifically, SSNI uses a pre-trained score network to estimate how much a data point deviates from the clean data distribution (i.e., score norms).
Then, based on the magnitude of score norms, SSNI applies a reweighting function to adaptively adjust $t^*$ for each sample, achieving sample-specific noise injections.
Empirically, incorporating our framework with existing DBP methods results in a notable improvement in both accuracy and robustness on CIFAR-10 and ImageNet-1K, highlighting the necessity to allocate \emph{distinct noise levels to different samples} in DBP methods.
Our code is available at: \url{https://github.com/tmlr-group/SSNI}.
\end{abstract}

\section{Introduction}
\emph{Deep neural networks} (DNNs) are vulnerable to adversarial examples, which is a longstanding problem in deep learning \citep{szegedy2014intriguing, goodfellow2015explaining}. 
Adversarial examples aim to mislead DNNs into making erroneous predictions by adding imperceptible adversarial noise to clean examples, which pose a significant security threat in critical applications \citep{dong2019efficient, cao2021invisible, jing2021too, bo2025trustworthy}.
To defend against adversarial examples, \emph{adversarial purification} (AP) stands out as a representative defensive mechanism, by leveraging pre-trained generative models to purify adversarial examples back towards their natural counterparts before feeding into a pre-trained classifier \citep{yoon2021adversarial, nie2022diffusion}.
Notably, AP methods benefit from their modularity, as the purifier operates independently of the downstream classifier, which facilitates seamless integration into existing systems and positions AP as a practical approach to improve the adversarial robustness of DNN-based classifiers.

Recently, \emph{diffusion-based purification} (DBP) methods have gained much attention as a promising framework in AP, which leverage the denoising nature of diffusion models to mitigate adversarial noise \citep{nie2022diffusion, xiao2023densepure, lee2023robust}. 
Generally, diffusion models train a forward process that maps from data distributions to simple distributions, e.g., Gaussian, and reverse this mapping via a reverse generative process \citep{ho2020denoising, song2021score}.
When applied in DBP methods, the forward process gradually injects Gaussian noise into the input sample, while the reverse process gradually purify noisy sample to recover the clean sample. 
The quality of the purified sample heavily depends on the amount of Gaussian noise added to the input during the forward process, which can be controlled by a noise level parameter $t^*$.
Existing DBP methods \citep{nie2022diffusion, xiao2023densepure, lee2023robust} manually select a constant $t^*$ for all samples.

 \begin{figure*}[!ht]
    \centering
    \begin{subfigure}{0.32\textwidth}
        \centering
        \includegraphics[width=\linewidth]{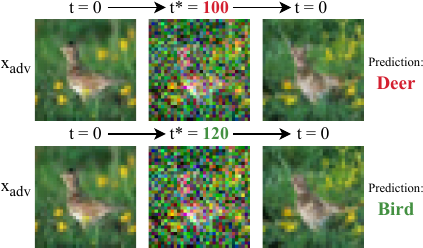}
        \caption{}
        \label{fig: motivation-3}
    \end{subfigure}
    \hfill
    \begin{subfigure}{0.32\textwidth}
        \centering
        \includegraphics[width=\linewidth]{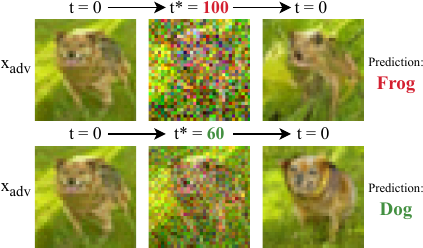}
        \caption{}
        \label{fig: motivation-2}
    \end{subfigure}
    \hfill
    \begin{subfigure}{0.33\textwidth}
        \centering
        \includegraphics[width=\linewidth]{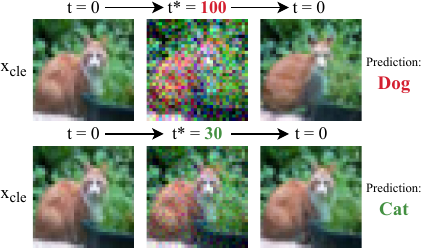}
        \caption{}
        \label{fig: motivation-1}
    \end{subfigure}
    \caption{For each sub-figure: the 1st column contains the input (i.e., could either be AEs or CEs), the 2nd column contains noise-injected examples with different $t^*$s, and the 3rd column contains purified examples. We use \emph{DiffPure} \citep{nie2022diffusion} with a sample-shared $t^* = 100$ selected by \citet{nie2022diffusion} to conduct this experiment on CIFAR-10 \citep{cifar}.
    The globally shared $t^* = 100$ offers a baseline, but results in suboptimal prediction performance compared to what could be achieved by tuning the noise level for individual samples.
    Notably, while the recovered images obtained by different noise levels may be visually indistinguishable, they carry different semantics. 
    For instance, the image is classified as ``frog'' (incorrect) with $t^*=100$ but as ``dog'' (correct) with $t^*=60$ (Figure~\ref{fig: motivation-2}). 
    These \emph{highlight the need for a sample-wise noise level adjustment}.}
    \label{fig: motivation}
\end{figure*}

However, we find that using a sample-shared $t^*$ may \emph{overlook} the fact that an optimal $t^*$ indeed could be different at sample-level, as demonstrated in Figure \ref{fig: motivation}. 
For example, in Figure \ref{fig: motivation-3}, $t^* = 100$ is too small, resulting in the adversarial noise not being sufficiently removed by the diffusion models.
This is because diffusion models are good at denoising samples that have been sufficiently corrupted by Gaussian noise through the forward process \citep{ho2020denoising, song2021score}. 
With a small $t^*$, the sample remains insufficiently corrupted, which limits the denoising capability in the reverse process and thereby compromising the robustness against adversarial examples.
On the other hand, in Figure \ref{fig: motivation-2} and \ref{fig: motivation-1}, $t^* = 100$ is too large, resulting in excessive disruption of the sample's semantic information during the forward process, which makes it difficult to recover the original semantics in the reverse process.
In this case, both robustness and clean accuracy are compromised, as the purified samples struggle to preserve the semantic consistency of clean samples.
These observations motivate us to make the \emph{first} attempt to adjust the noise level on a sample-specific basis.

In this paper, we propose \emph{\textbf{S}ample-specific \textbf{S}core-aware \textbf{N}oise \textbf{I}njection} (SSNI), a new framework that leverages the distance of a sample from the clean data distribution to adaptively adjust $t^*$ on a sample-specific basis. SSNI aims to inject less noise to cleaner samples, and vice versa. 

To implement SSNI, inspired by the fact that \emph{scores} (i.e., $\nabla_{\rvx}{\log p_t(\rvx)}$) reflect the directional momentum of samples toward the high-density areas of clean data distribution \citep{song2019generative}, we use \emph{score norms} (i.e., $\left\|\nabla_{\rvx}{\log p_t(\rvx)}\right\|$) as a natural metric to measure the deviation of a data point from the clean data distribution. 
In Section \ref{Sec: motivation}, we establish the relationship between the score norm and the noise level required for different samples.
Specifically, samples with different score norms tend to have accumulated different noise levels. 
Furthermore, we empirically show that the cleaner samples -- those closer to the clean data distribution -- exhibit lower score norm, justifying the rationale of using score norms for reweighting $t^*$.
Concretely, we use a pre-trained score network to estimate the score norm for each sample.
Based on this, we propose two reweighting functions that adaptively adjust $t^*$ according to its score norm, achieving sample-specific noise injections (see Section \ref{sec: realization of SSNI}). Notably, this reweighting process is \emph{lightweight}, ensuring that SSNI is computationally feasible and can be applied in practice with minimal overhead (see Section \ref{sec: inference time}).

Through extensive evaluations on benchmark image datasets such as CIFAR-10 \citep{cifar} and ImageNet-1K \citep{deng2009imagenet}, we demonstrate the effectiveness of SSNI in Section \ref{Sec: experiments}. 
Specifically, combined with different DBP methods \citep{nie2022diffusion, xiao2023densepure, lee2023robust}, SSNI can boost clean accuracy and robust accuracy \emph{simultaneously} by a notable margin against the well-designed adaptive white-box attack (see Section \ref{Sec: performance evaluation}). 

The success of SSNI takes root in the following aspects: (1) an optimal noise level $t^*$ for each sample indeed could be different, making SSNI an effective approach to unleash the intrinsic strength of DBP methods; (2) existing DBP methods often inject excessive noise into clean samples, resulting in a degradation in clean accuracy. By contrast, SSNI injects less noise to clean samples, and thereby notably improving the clean accuracy. Meanwhile, SSNI can effectively handle adversarial samples by injecting sufficient noise on each sample; (3) SSNI is designed as general framework instead of a specific method, allowing it to be seamlessly integrated with a variety of existing DBP methods.

\section{Preliminary and Related Work}
\label{Sec: Preliminary and Related Work}
In this section, we first review diffusion models and scores in detail. We then review the related work of DBP methods. Detailed related work can be found in Appendix \ref{A: related work}.

\textbf{Diffusion models} are generative models designed to approximate the underlying clean data distribution $p(\rvx_0)$, by learning a parametric distribution $p_{\phi}(\rvx_0)$ with a {\it forward} and a {\it reverse} process.
In \emph{Denoising Diffusion Probabilistic Models} (DDPM) \citep{ho2020denoising}, the {\it forward} process of a diffusion model defined on $\gX \subseteq \sR^d$ can be expressed by:
\begin{equation}\label{eq: forward}
    q(\rvx_t | \rvx_0) = \gN \left( \rvx_t; \sqrt{\bar{\alpha}} \rvx_0, (1 - \bar{\alpha}_t) \rmI \right),
\end{equation}
where $\bar{\alpha}_t = \prod^t_{i=1}(1-\beta_i)$ and $\{\beta_t\}_{t \in [0, T]}$ are predefined noise scales with $\beta_t \in (0, 1)$ for all $t$. As $t$ increases, $\rvx_t$ converges toward isotropic Gaussian noise.
In the {\it reverse} process, DDPM seeks to recover the clean data from noise by simulating a Markov chain in the reverse direction over $T$ steps. 
The reverse transition at each intermediate step is modeled by
\begin{equation}\label{eq: reverse}
    p_{\phi}(\rvx_{t-1}|\rvx_t) = \mathcal{N}(\rvx_{t-1}; \mu_{\phi}(\rvx_t, t), \sigma_t^2 \mathbf{I}),
\end{equation}
where $\mu_{\phi}(\rvx_t, t)=\frac{1}{\sqrt{1-\beta_t}} \left( \rvx_t - \frac{\beta_t}{\sqrt{1 - \bar{\alpha}_t}} \boldsymbol{\epsilon}_{\phi}(\rvx_t, t) \right)$ is the predicted mean at $t$, and $\sigma_t$ is a fixed variance \citep{ho2020denoising}.
More specifically, the model learns to predict the noise $\boldsymbol{\epsilon}_{\phi^*}(\rvx_t, t)$ added to the data.
The training objective minimizes the distance between the true and predicted noise, accounting for different noise levels controlled by $t$:
\begin{equation}
    \boldsymbol{\phi}^* = \arg\min_{\boldsymbol{\phi}} \E_{\rvx_0, t, \boldsymbol{\epsilon}} \left[ \left\lVert \boldsymbol{\epsilon} - \boldsymbol{\epsilon_{\phi}} \left( \sqrt{\bar{\alpha}_t} \rvx_0 + \sqrt{1 - \bar{\alpha}_t} \boldsymbol{\epsilon}, t \right) \right \rVert_2^2 \right]. \nonumber
\end{equation}

The generative process starts from pure noise $\rvx_T \sim \gN(\mathbf{0}, \rmI)$ and progressively removes noise through $T$ steps, with a random noise term $\boldsymbol{\epsilon} \sim \gN(\mathbf{0}, \rmI)$ incorporated at each step:
\begin{equation*}
    \hat{\rvx}_{t-1} = \frac{1}{\sqrt{1 - \beta_t}} \left(\hat{\rvx}_t - \frac{\beta_t}{\sqrt{1-\bar{\alpha}_t}} \boldsymbol{\epsilon}_{\phi^*}(\hat{\rvx}_t, t) \right) + \sqrt{\beta_{t}}\boldsymbol{\epsilon}.
    \label{eq: DDPM Reverse Sampling}
\end{equation*}

\textbf{Score and score norm.} 
In diffusion models, \emph{score} refers to the gradient of the log-probability density $\nabla_{\rvx}{\log p_t(\rvx)}$, which maps the steepest ascent direction of the log-density in the probability space~\citep{song2019generative}.
The associated \emph{score norm} $\left\|\nabla_{\rvx}{\log p(\rvx)}\right\|$ measures the gradient's magnitude, reflecting how much a data point deviates from the clean data distribution: points located in low-probability regions typically exhibit larger score norms, while those situated closer to high-density regions of clean data manifest smaller score norms. 
This property is valuable for adversarial example detection, as demonstrated by \citet{yoon2021adversarial} who establish that score norms can effectively distinguish between adversarial and clean examples. However, direct computation of the score is often intractable for complex data distributions (e.g., images). 
In practice, it is estimated using neural networks $s_{\theta}: \gX \times \sR_{+} \to \gX$ that take as input both a data point $\rvx$ and a specified timestep $t^{\rm S} \in \sR_{+}$ to approximate $\nabla_{\rvx}{\log p_t(\rvx)}$ by minimizing the Fisher divergence between the true and estimated scores.
Also, $s_{\theta}(\rvx, t^{\rm S})$ enables efficient computation of $\left\|\nabla_{\rvx}{\log p(\rvx)}\right\|$.


\textbf{Diffusion-based purification.} \emph{Adversarial purification} (AP) leverages generative models as an add-on module to purify adversarial examples before classification. Within this context, \emph{diffusion-based purification} (DBP) methods have emerged as a promising framework, exploiting the inherent denoising nature of diffusion models to filter out adversarial noise \citep{nie2022diffusion, wang2022guided, xiao2023densepure, lee2023robust}. 
Specifically, DBP works with \emph{pre-trained} diffusion model by first corrupting an adversarial input through the forward process and then iteratively reconstructing it via the reverse process.
This projects the input back onto the clean data manifold while stripping away adversarial artifacts.
\citet{wang2022guided} introduce input guidance during the reverse diffusion process to ensure the purified outputs stay close to the inputs. 
\citet{lee2023robust} propose a fine-tuned gradual noise scheduling for multi-step purifications. 
\citet{bai2024diffusion} improve the reverse diffusion process by incorporating contrastive objectives.

\begin{figure}[t]
    \centering
    \includegraphics[width=\linewidth]{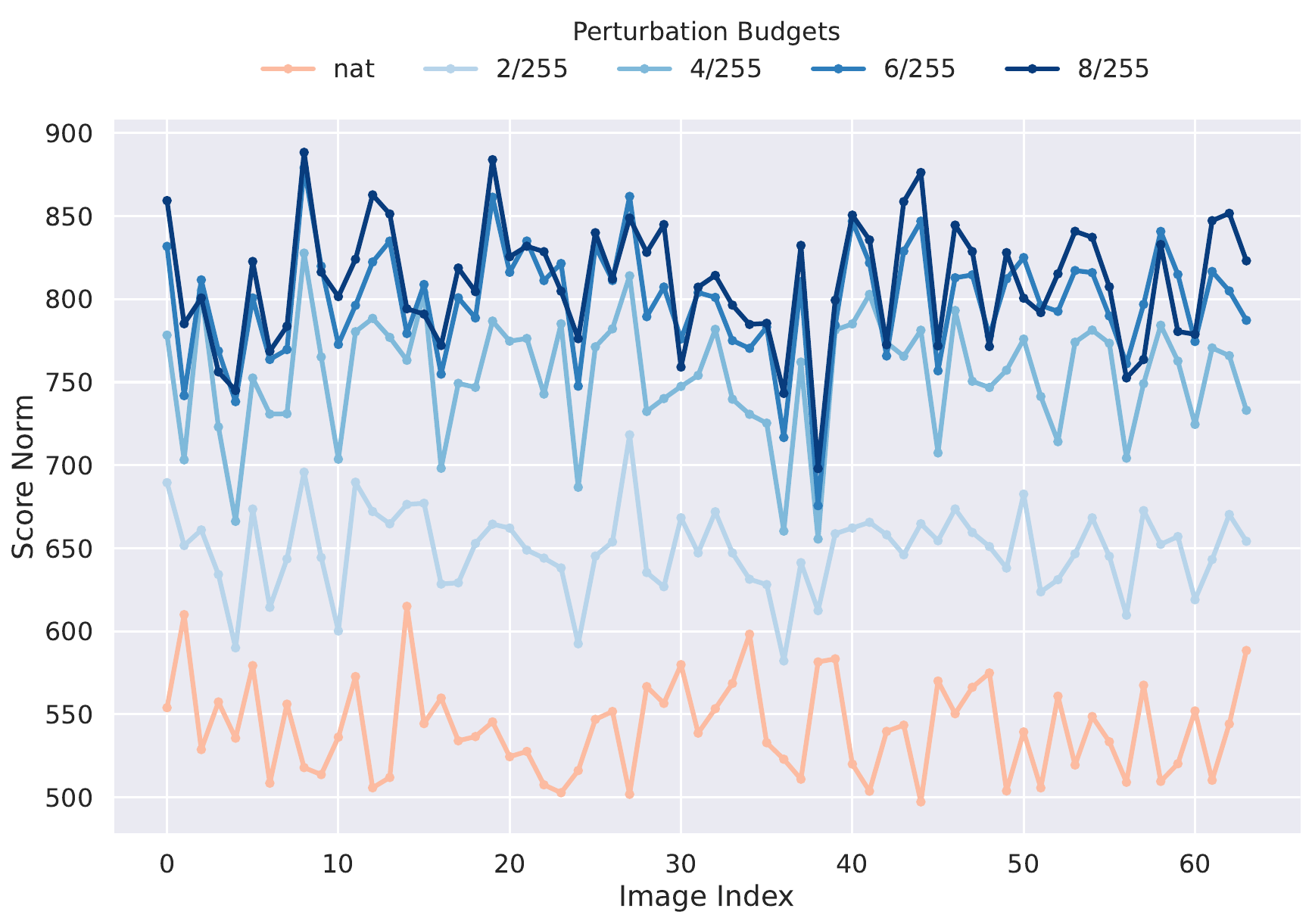}
    \caption{Relationship between score norms and perturbation budgets. We use one batch of clean data from \emph{CIFAR-10} and employ PGD+EOT $\ell_\infty (\epsilon = 8/255)$ as the attack.}
    \label{fig: eps norm with different perturbations}
\end{figure}

\section{Motivation}
\label{Sec: motivation}
We now elaborate on the motivation of our method by connecting the impact of different perturbation budgets $\epsilon$ to the required noise level $t^{*}$ of each sample through score norm.

\textbf{Sample-shared noise-level $t^*$ fails to address diverse adversarial perturbations.}
We empirically observe that an optimal noise level $t^*$ for each sample indeed could be different:
While a constant noise level $t^* = 100$, as suggested by \citet{nie2022diffusion}, yields strong performances for certain samples, it leads to \emph{suboptimal} results for others (from Figure~\ref{fig: motivation}). 
Since DBP relies on adequate $t^{*}$ for forward noise injection to remove adversarial perturbation, a shared $t^*$ cannot adapt to the distinct adversarial perturbations of individual examples, leading to a suboptimal accuracy-robustness trade-off.
Specifically, Figure \ref{fig: motivation-3} shows that $t^* = 100$ is insufficient to remove the adversarial noise, leaving residual perturbations that compromise robustness.
In contrast, $t^* = 100$ overly suppresses the other sample's semantic information during the forward process, making it difficult to recover the original semantics via the reverse process (Figure \ref{fig: motivation-2}-\ref{fig: motivation-1}).
These findings highlight the need for \emph{sample-specific noise injection levels tailored to individual adversarial perturbations}.

\begin{figure*}[t]
    \centering
    \includegraphics[width=\linewidth]{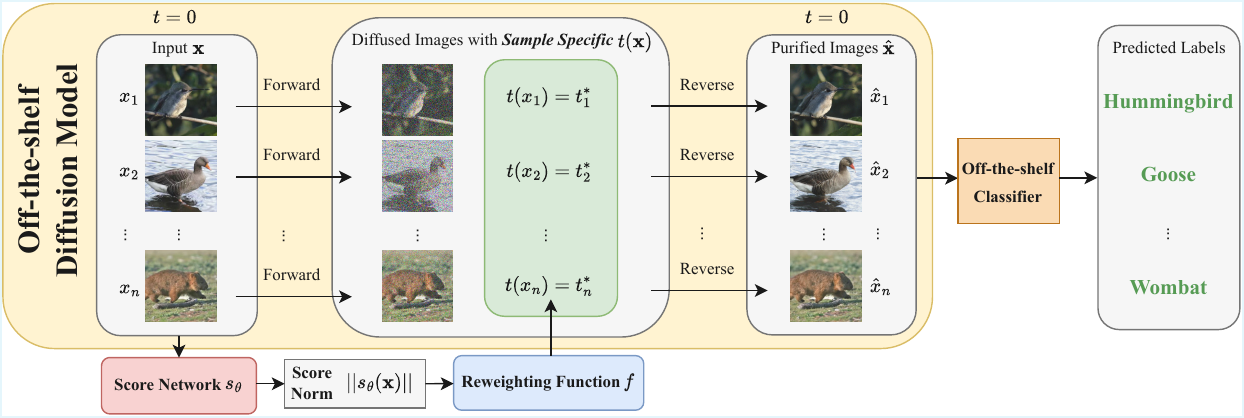}
    \caption{An overview of the proposed SSNI framework. SSNI introduces a novel sample-specific mechanism to adaptively adjust the noise injection level for each sample, enhancing purification effectiveness in DBP methods. The process begins by forwarding each sample image $x_i$ through the forward diffusion process using an off-the-shelf diffusion model. To determine how much noise to inject into each sample, SSNI employs a pre-trained score network $s_\theta$ to compute the score norm $\|s_\theta(x_i)\|$, which reflects the distance of the sample from the clean data distribution. Based on this score norm, a reweighting function $f$ adaptively determines the optimal noise level $t^*_i$ for each sample. Finally, each sample is purified through a reverse diffusion process before being classified. Notably, SSNI is designed as a general framework rather than a specific method, which can be seamlessly integrated with a wide range of existing DBP methods.}
    \label{fig:pipeline}
\end{figure*}

\textbf{Score norms vary across perturbation budgets.}
Adversarial and clean examples are from distinct distributions \citep{gao2021maximum}.
Motivated by this, we further investigate how different perturbation budgets $\epsilon$ affect score norms under adversarial attacks (Figure~\ref{fig: eps norm with different perturbations}).
Specifically, we compute the score norm of different samples undergoing PGD+EOT $\ell_\infty (\epsilon = 8/255)$ with perturbation budgets varying between $0$ and $8/255$ on CIFAR-10. 
We observe a consistent pattern:
score norms scale directly with perturbation strength: larger $\epsilon$ values lead to higher norms, whereas smaller perturbations (i.e., cleaner samples) lead to lower norms
The findings extend the role of score norms from differentiating adversarial/clean samples~\citep{yoon2021adversarial} to \emph{differentiate adversarial examples based on their perturbation strength}.

\textbf{Different score norms imply sample-specific $t^{*}$.}
Higher score norms signal greater deviation from the clean data, often caused by larger $\epsilon$.
Intuitively, \emph{samples with elevated score norms demand higher} $t^{*}$ (i.e., more aggressive noise injection) to remove adversarial patterns.
This dependency creates a link: score norms can act as \emph{proxies} for estimating the optimal sample-specific $t^*$. 
Doing so successfully leads to sufficient purification for adversarial examples while preserving fidelity for cleaner inputs.

\section{Sample-specific Score-aware Noise Injection}\label{Sec: method}
Motivated by Section \ref{Sec: motivation}, we propose \emph{\textbf{S}ample-specific \textbf{S}core-aware \textbf{N}oise \textbf{I}njection} (SSNI), a generalized DBP framework that adaptively adjusts the noise level for each sample based on how much it deviates from the clean data distribution, measured by the score norm.
We begin by introducing the SSNI framework, followed by a connection with existing DBP methods and the empirical realization of SSNI.

\subsection{Framework of SSNI}
\label{sec: framework of SSNI}
\textbf{Overview}.
SSNI builds upon existing DBP methods by \emph{reweighting} the optimal noise level $t^*$ from a global, sample-shared constant to a sample-specific quantity. 
At its core, SSNI leverages score norms to modulate the noise injected into each sample during diffusion, ensuring a more targeted denoising process tailored to each sample.
We visually illustrate SSNI in Figure \ref{fig:pipeline}, and describe the algorithmic workflow in Algorithm \ref{alg1:SSNI}.

\textbf{DBP with sample-shared noise level $t^*$}. 
Existing DBP methods use an off-the-shelf diffusion model for data purification, and a classifier responsible for label prediction.
Let $\gY$ be the label space for the classification task.
Denote the forward diffusion process by $D: \gX \to \gX$, the reverse process by $R: \gX \to \gX$, and the classifier by $C: \gX \to \gY$.
For classification, each adversarial sample goes through
\begin{equation}
    h(\rvx) = C \circ R \circ D (\rvx), \quad \text{with } \rvx = \rvx_0.
\end{equation}
In this context, $\rvx_T = D(\rvx_0)$ refers to the noisy image obtained after $T$ steps of diffusion, and $\hat{\rvx}_0 = R(\rvx_T)$ represents the corresponding recovered images through the reverse process.
Specifically, these methods predetermine a {\it constant} noise level $t^*$ for all samples, following a shared noise schedule $\{ \beta_t \}_{t \in [0, T]}$.
The outcome of the forward process defined in \eqref{eq: forward} can be expressed as:
\begin{equation*}
    \rvx_T = \sqrt{\prod\nolimits^{t^*}_{i =1}(1 - \beta_i)}\rvx + \sqrt{1 - \prod\nolimits^{t^*}_{i =1}(1 - \beta_i)}\mathbf{\boldsymbol{\epsilon}},
\end{equation*}
for $\rvx=\rvx_0, \, \forall \rvx \in \gX$,
where $\rvx_0$ represents the original data, and $\boldsymbol{\epsilon} \sim \mathcal{N}(\mathbf{0}, \mathbf{I})$ denotes the Gaussian noise.

\textbf{From sample-shared to sample-specific noise level}.
SSNI takes a step further by transforming the sample-shared noise level $T_{\rm SH}(\rvx) = t^*$ into a \emph{sample-specific} noise level $T_{\rm SI}(\rvx) = t(\rvx)$, computed through
\begin{equation}\label{eq: sample specific noise level}
    t(\rvx) = f(\left\|s_\theta(\rvx, t^{\rm S})\right\|,~t^*),
\end{equation}
using a \emph{pre-trained} score network $\theta$,
where $s_\theta(\rvx, t^{\rm S})$ is the score evaluated at an \emph{arbitrary} reference noise level $t^{\rm S}$, and $f(\cdot, \cdot)$ is a reweighting function (detailed in Section~\ref{sec: realization of SSNI}) that adjusts $t^*$ based on the score norm.
The forward process of SSNI then takes the form:
\begin{equation}
\label{eq: sample-specific noise infusion}
    \rvx_{t(\rvx)} = \sqrt{\prod\nolimits^{t(\rvx)}_{i =1}(1 - \beta_i)}\rvx + \sqrt{1 - \prod\nolimits^{t(\rvx)}_{i =1}(1 - \beta_i)}\mathbf{\boldsymbol{\epsilon}},
\end{equation}
for $\rvx=\rvx_0, \, \forall \rvx \in \gX$.
In this way, SSNI establishes a link between the sample's deviation from the clean data and the intensity of noise injected during diffusion.

\subsection{Unifying Sample-shared and Sample-specific DBP}\label{sec: unifying SSNI and Sample-shared DBP}
We define a generalized purification operator encompassing both sample-shared and sample-specific noise based DBP methods as $\Phi(\rvx) = R(\rvx_{T(\rvx)})$, where $R$ denotes the reverse process, $\rvx_{T(\rvx)}$ is the noisy version of $\rvx$ after $T(\rvx)$ steps of diffusion, and $T: \gX \to \gT$ is a function that determines the noise level for each input, with $\gT = [0, T_{\rm max}]$ being the range of possible noise levels.
Based on $\Phi(\cdot)$, we derive the following understandings (justifications are in Appendix~\ref{A: justification}).

\textbf{Sample-shared DBP is a special case of SSNI}.
The noise level of a sample-shared DBP, denoted by $T_{\rm SH}(\rvx) \triangleq t^*$, is a constant for $\forall \rvx \in \gX$, while for SSNI, we have: $T_{\rm SI}(\rvx) \triangleq t(\rvx) = f(\left\|s_\theta(\rvx, t^{\rm S})\right\|, t^*)$.
Clearly, any $T_{\rm SH}(\rvx)$ can be expressed by $T_{\rm SI}(\rvx)$, implying that any sample-shared noise level $t^*$ is equivalently represented by SSNI with a constant reweighting function.

\textbf{SSNI has higher purification flexibility}.
To compare the purification capabilities of different DBP strategies, we introduce the purification range $\Omega$. 
For an input $\rvx \in \gX$, we define $\Omega(\rvx) = \left\{ \Phi(\rvx) \mid \Phi(\rvx) = R(\rvx_{\tau(\rvx)}), \tau: \gX \to \gT \right\}$, which characterizes all possible purified outputs that a DBP strategy can generate for $\rvx$ when using different noise levels $\tau(\rvx)$.
We find that $\Omega_{\rm SH} \subseteq \Omega_{\rm SI}$ holds for any $\rvx \in \gX$, and there exists at least one $\rvx \in \gX$ for which the inclusion is strict, i.e., $\Omega_{\rm SH} \subsetneq \Omega_{\rm SI}$.
These results show that SSNI expands the space of possible purified outputs beyond what sample-shared DBP can achieve, thus enabling {\it greater flexibility} in the purification process.


\input{algorithm/DBP-SSNI}
\subsection{Realization of SSNI}
\label{sec: realization of SSNI}

We now detail the empirical realizations of SSNI. 

\textbf{Realization of the score.}
A challenge to obtaining sample-specific noise levels lies in the score network's dependency on a score-evaluation noise level as input.
Estimating scores of different adversarial samples at a single fixed $t^{\rm S}$ introduces sensitivity: scores computed at different reference levels (e.g., $t^{\rm S}_i \neq t^{\rm S}_j)$ yield inconsistent norms, biasing the selection of $t^*$.
Crucially, the ``true'' optimal $t^{\rm S}$ for each sample is unknown a priori, which leads to a circular dependency between reference noise level selection and score estimation.
To alleviate this, we leverage \emph{expected perturbation score} (EPS), which aggregates scores of perturbed samples across a spectrum of noise levels rather than relying on a single $t^{\rm S}$.
This integration reduces sensitivity to individual $t^{\rm S}$ choices~\citep{zhang2023detecting}.
EPS is defined as
\begin{equation}
    {\rm EPS}(\rvx) = \E_{t \sim U(0, t^{\rm S})} \nabla_{\rvx} \log p_t(\rvx),
    \label{eq: EPS definition}
\end{equation}
where $p_t(\rvx)$ is the marginal probability density and $t^{\rm S}$ is the maximum noise level for EPS. EPS computes the expectation of the scores of perturbed images across different noise levels $t \sim U(0, t^{\rm S})$, making it more robust to the changes in noise levels.
Notably, this $t^{\rm S}$ is \textit{different} from the sample-shared noise injection level $t^*$.
We will omit $t^{\rm S}$ from the notation of ${\rm EPS}(\rvx)$ for brevity hereafter.
Following \citet{zhang2023detecting}, we set $t^{\rm S} = 20$.
In practice, a score $\nabla_\rvx \log_{p_t}(\rvx)$ can be approximated by a score network. 
Specifically, we employ a score network pre-trained using the score matching objective \citep{song2019generative}. 

\textbf{Realization of the linear reweighting function.} We first design a linear function to reweight $t^*$:
\begin{equation}
    f_{\text{linear}}(\left\|\text{EPS}(\rvx)\right\|,~t^*) = \frac{\left\|\text{EPS}(\rvx)\right\| - \xi_{\min}}{\xi_{\max} - \xi_{\min}} \times t^* + b,
    \label{eq: linear reweighting function}
\end{equation}
where $b$ is a bias term and $t^*$ denotes the optimal sample-shared noise level selected by \citet{nie2022diffusion}.
To implement this reweighting function, we extract 5,000 validation clean examples from the training data (denoted as  $\rvx_v$) and we use $\left\|\text{EPS}(\rvx_v)\right\|$ as a \emph{reference} to indicate the approximate EPS norm values of clean data, which can help us reweight $t^*$. 
Then we define $\xi_{\min} = \min(\left\|\text{EPS}(\rvx)\right\|,~\left\|\text{EPS}(\rvx_v)\right\|)$ and $\xi_{\max} = \max(\left\|\text{EPS}(\rvx)\right\|,~\left\|\text{EPS}(\rvx_v)\right\|)$
to normalize $\left\|\text{EPS}(\rvx)\right\|$ such that the coefficient of $t^*$ is within a range of $[0, 1]$, ensuring that the reweighted $t^*$ stays positive and avoids unbounded growth, thus preserving the semantic information.

\input{tables/CIFAR10-MAIN}

\textbf{Realization of the non-linear reweighting function.} We then design a non-linear function based on the sigmoid function, which has two horizontal asymptotes:
\begin{equation}
    f_\sigma(\left\|\text{EPS}(\rvx)\right\|,~t^*) = \frac{t^* + b}{1 + \exp\{-(\left\|\text{EPS}(\rvx)\right\|-\mu)/\tau\}},
    \label{eq: non-linear reweighting function}
\end{equation}
where $b$ is a bias term and $t^*$ denotes the optimal sample-shared noise level selected by \citet{nie2022diffusion} and $\tau$ is a temperature coefficient that controls the sharpness of the function. We denote the mean value of $\left\|\text{EPS}(\rvx_v)\right\|$ as $\mu$. This ensures that when the difference between $\left\|\text{EPS}(\rvx)\right\|$ and $\mu$ is large, the reweighted $t^*$ can approach to the maximum $t^*$ in a more smooth way, and vice versa.

\textbf{Adding a bias term to the reweighting function.} One limitation of the above-mentioned reweighting functions is that \emph{the reweighted $t^*$ cannot exceed the original $t^*$}, which may result in some adversarial noise not being removed for some adversarial examples. To address this issue, we introduce an extra bias term (i.e., $b$) to the reweighting function, which can increase the upper bound of the reweighted $t^*$ so that the maximum possible reweighted $t^*$ can exceed original $t^*$. Empirically, we find that this can further improve the robust accuracy without compromising the clean accuracy.

\section{Experiments}
\label{Sec: experiments}
In this section, we use \emph{SSNI-L} to denote our method with the \emph{linear} reweighting function, and use \emph{SSNI-N} to denote our method with the \emph{non-linear} reweighting function. 

\subsection{Experimental Settings}
\label{Sec: experimental settings}

\textbf{Datasets and model architectures.} We consider two datasets for our evaluations: CIFAR-10 \citep{cifar}, and ImageNet-1K \citep{deng2009imagenet}.
For classifiers, we use the pre-trained WideResNet-28-10 and WideResNet-70-16 for CIFAR-10, and the pre-trained ResNet-50 for ImageNet-1K.
For diffusion models, we employ two off-the-shelf diffusion models trained on CIFAR-10 and ImageNet-1K~\citep{song2021score,dhariwal2021diffusion}.

\textbf{Evaluation metrics.} For all experiments, we consider the standard accuracy (i.e., accuracy on clean examples) and robust accuracy (i.e., accuracy on adversarial examples) as the evaluation metrics.

\textbf{Baseline settings.} We use three well-known DBP methods as our baselines: \emph{DiffPure} \citep{nie2022diffusion}, \emph{GDMP} \citep{wang2022guided} and \emph{GNS} \citep{lee2023robust}. The detailed configurations can be found in Appendix~\ref{Defense Methods Configurations}. For the reverse process within diffusion models, we consider DDPM sampling method \citep{ho2020denoising} in the DBP methods.

\textbf{Evaluation settings for DBP baselines.} Following \citet{lee2023robust}, we use a fixed subset of 512 randomly sampled images for all evaluations due to high computational cost of applying adaptive white-box attacks to DBP methods.
\citet{lee2023robust} provide a robust evaluation framework for existing DBP methods and demonstrate that PGD+EOT \citep{madry2018towards, anish2018synthesizing} is the golden standard for DBP evaluations.
Therefore, following \citet{lee2023robust}, we mainly use adaptive white-box PGD+EOT attack with 200 PGD iterations for CIFAR-10 and 20 PGD iterations for ImageNet-1K.
We use 20 EOT iterations for all experiments to mitigate the stochasticity introduced by the diffusion models. 
As PGD is a gradient-based attack, we compute the gradients of the entire process from a surrogate process. 
The details of the surrogate process is explained in Appendix~\ref{Surrogate Process of Gradient Computation}.
We also evaluate DBP methods under adaptive BPDA+EOT attack \citep{athalye2018obfuscated}, which leverages an identity function to approximate the direct gradient rather than direct computing the gradient of the defense system.

\textbf{Evaluation settings for SSNI.} Since SSNI introduces an extra reweighting process than DBP baselines, we implicitly design two adaptive white-box attacks by considering the \emph{entire defense mechanism} of SSNI (i.e., adaptive white-box PGD+EOT attack and adaptive white-box BPDA+EOT attack).
\emph{To make a fair comparison, we evaluate SSNI on adaptive white-box attacks with the same configurations mentioned above}.  
The algorithmic descriptions for the adaptive white-box PGD+EOT attack and adaptive white-box BPDA+EOT attack is provided in Appendix \ref{A: adaptive pgd+eot} and \ref{BPDA+EOT Adaptive Attack}.
In addition, to evaluate the generalization and adaptability of SSNI to diverse adversarial attacks, we further include AutoAttack \citep{croce2020reliable}, DiffAttack \citep{chen2023diffusion} and the Diff-PGD attack \citep{xue2024diffusion} in Section \ref{Sec: additional attacks}. We set the iteration number to 5 for Diff-PGD.

\subsection{Defending Against Adaptive White-box PGD+EOT}
\label{Sec: performance evaluation}
We mainly present and analyze the evaluation results of \emph{SSNI-N} in this section and the experimental results of \emph{SSNI-L} can be found in Appendix~\ref{Sec: Performance Evaluation of SSNI-L}.

\input{tables/ImageNet-SUB}

\input{tables/BPDA}

\begin{table*}[t]
\centering
\caption{Standard and robust accuracy (\%) against AutoAttack (random version), DiffAttack and Diff-PGD attack with $\ell_\infty (\epsilon = 8/255)$ on \emph{CIFAR-10}. We report mean and standard deviation over three runs. We show the most successful defense in \textbf{bold}.}
\vspace{2pt}
\label{table:multi-attack-2810}
\resizebox{0.8\linewidth}{!}{%
\begin{tabular}{clcccc}
\toprule
\multicolumn{6}{c}{$\ell_\infty~(\epsilon = 8/255)$}\\
\midrule
& DBP Method & Standard & AutoAttack & DiffAttack & Diff-PGD \\
\midrule
\multirow{7}{*}{\rotatebox[origin=c]{90}{WRN-28-10}} & \citet{nie2022diffusion} & 89.71$\pm$0.72& 66.73$\pm$0.21 & 47.16$\pm$0.48 & 54.95$\pm$0.77\\
& \cellcolor{lg}{+ \emph{SSNI-N}} & \cellcolor{lg}{\bf{93.29$\pm$0.37 \textcolor{dg}{(+3.58)}}} & \cellcolor{lg}{\bf{66.94$\pm$0.44 \textcolor{dg}{(+0.21)}}} & \cellcolor{lg}{\bf{48.15$\pm$0.22 \textcolor{dg}{(+0.99)}}} & \cellcolor{lg}{\bf{56.10$\pm$0.35 \textcolor{dg}{(+1.15)}}}\\
\cmidrule{2-6}
& \citet{wang2022guided} & 92.45$\pm$0.64 & 64.48$\pm$0.62 & 54.27$\pm$0.72 & 41.45$\pm$0.60\\
& \cellcolor{lg}{+ \emph{SSNI-N}} & \cellcolor{lg}{\bf{94.08$\pm$0.33 \textcolor{dg}{(+1.63)}}} & \cellcolor{lg}{\bf{66.53$\pm$0.46 \textcolor{dg}{(+2.05)}}} & \cellcolor{lg}{\bf{55.81$\pm$0.33 \textcolor{dg}{(+1.54)}}} & \cellcolor{lg}{\bf{42.91$\pm$0.56 \textcolor{dg}{(+1.46)}}}\\
\cmidrule{2-6}
& \citet{lee2023robust} & 90.10$\pm$0.18 & 69.92$\pm$0.30 
 & 56.04$\pm$0.58 & 59.02$\pm$0.28\\
& \cellcolor{lg}{+ \emph{SSNI-N}} & \cellcolor{lg}{\bf{93.55$\pm$0.55 \textcolor{dg}{(+3.45)}}} & \cellcolor{lg}{\bf{72.27$\pm$0.19 \textcolor{dg}{(+2.35)}}} & \cellcolor{lg}{\bf{56.80$\pm$0.41 \textcolor{dg}{(+0.76)}}} & \cellcolor{lg}{\bf{61.43$\pm$0.58 \textcolor{dg}{(+2.41)}}}\\
\bottomrule
\end{tabular}
}
\end{table*}
\input{tables/Diff_sampling_methods}

\textbf{Result analysis on CIFAR-10.} Table \ref{table:best_cifar10} shows the standard and robust accuracy against PGD+EOT $\ell_{\infty}(\epsilon=8/255)$ and $\ell_{2}(\epsilon=0.5)$ threat models on CIFAR-10, respectively.
Notably, \emph{SSNI-N} effectively improves the accuracy-robustness trade-off on PGD+EOT $\ell_{\infty}(\epsilon=8/255)$ compared to DBP baselines.
Specifically, \emph{SSNI-N} improves standard accuracy of \emph{DiffPure} by $3.58\%$ on WideResNet-28-10 and WideResNet-70-16 without compromising robust accuracy.
For \emph{GDMP}, the standard accuracy grows by $1.63\%$ on WideResNet-28-10 and by $2.47\%$ on WideResNet-70-16, respectively.
Notably, \emph{SSNI-N} improves the robust accuracy of \emph{GDMP} by $4.23\%$ on WideResNet-28-10 and by $2.48\%$ on WideResNet-70-16.
For \emph{GNS}, both the standard accuracy and robust accuracy are improved by a notable margin.
We can observe a similar trend in PGD+EOT $\ell_{2}(\epsilon=0.5)$. Despite some decreases in robust accuracy (e.g., 0.06\% on \emph{DiffPure} and 0.39\% on \emph{GDMP}), \emph{SSNI-N} can improve standard accuracy by a notable margin, and thus improving accuracy-robustness trade-off.

\textbf{Result analysis on ImageNet-1K.} Table~\ref{table:linf_imagenet_resnet50} presents the evaluation results against adaptive white-box PGD+EOT $\ell_{\infty} (\epsilon = 4/255)$ on ImageNet-1K.
\emph{SSNI-N} outperforms all baseline methods by notably improving both the standard and robust accuracy, which demonstrates the effectiveness of \emph{SSNI-N} in defending against strong white-box adaptive attack and indicates the strong scalability of SSNI on large-scale datasets such as ImageNet-1K.

\subsection{Defending Against Adaptive White-box BPDA+EOT}
\label{Sec: Defense Against BPDA+EOT}
We mainly present and analyze the evaluation results of \emph{SSNI-N} in this section and the experimental results of \emph{SSNI-L} can be found in Appendix~\ref{Sec: Performance Evaluation of SSNI-L}.

We further evaluate the performance of \emph{SSNI-N} against adaptive white-box BPDA+EOT $\ell_{\infty}(\epsilon=8/255)$, which is an adaptive attack specifically designed for DBP methods \citep{tramer2020on,hill2021stochastic}, as demonstrated in Table~\ref{table:bpda2810}. Specifically, incorporating \emph{SSNI-N} with \emph{DiffPure} can further improve the standard accuracy by 3.58\% without compromising robust accuracy.
Notably, incorporating \emph{SSNI-N} with \emph{GDMP} can improve the standard and robust accuracy \emph{simultaneously} by a large margin.
Despite some decreases in robust accuracy when incorporating \emph{SSNI-N} with \emph{GNS} (i.e., 1.10\%), \emph{SSNI-N} can improve standard accuracy significantly (i.e., 3.45\%), and thus improving the accuracy-robustness trade-off by a notable margin.

\vspace{-5pt}
\subsection{Defending Against Additional Attacks}
\label{Sec: additional attacks}
Table~\ref{table:multi-attack-2810} reports the standard and robust accuracy of various DBP methods on CIFAR-10 under three additional white-box attacks: AutoAttack, DiffAttack, and Diff-PGD, all under the $\ell_{\infty}(\epsilon = 8/255)$ threat model.
For the \emph{DiffPure}, \emph{SSNI-N} improves standard accuracy by 3.58\% and brings robustness gains under AutoAttack (i.e., 0.21\%), DiffAttack (i.e., 0.99\%), and Diff-PGD (i.e., 1.15\%).
For \emph{GDMP}, \emph{SSNI-N} increases standard accuracy by 1.63\%, and achieves significant robustness improvements under AutoAttack (i.e., 2.05\%) and DiffAttack (i.e., 1.54\%), with a slight gain under Diff-PGD (i.e., 1.46\%).
Regarding \emph{GNS}, \emph{SSNI-N} boosts the standard accuracy by 3.45\%, and the robust accuracy improves under AutoAttack (i.e., 2.35\%), DiffAttack (i.e., 1.70\%), and Diff-PGD (i.e., 2.41\%), respectively.
Overall, \emph{SSNI-N} consistently enhances both standard and robust accuracy across all three attack types, demonstrating its strong generalization and adaptability to diverse adversarial threats.

\subsection{Ablation Study}
\label{Sec: ablation study}

\textbf{Ablation study on $\tau$ in \emph{SSNI-N}.} We investigate how the temperature coefficient $\tau$ in \eqref{eq: non-linear reweighting function} affects the performance of \emph{SSNI-N} against adaptive white-box PGD+EOT $\ell_{\infty}(\epsilon=8/255)$ attack on CIFAR-10 in Figure~\ref{fig: ablation linf}.
The temperature coefficient $\tau$ controls the sharpness of the curve of the \emph{non-linear} reweighting function.
A higher $\tau$ leads to a more smooth transition between the low and high values of the reweighting function, resulting in less sensitivity to the changes of the input.
From Figure~\ref{fig: ablation linf}, the standard accuracy remains stable across different $\tau$s, while the robust accuracy increases to the climax when $\tau = 20$.
Therefore, we choose $\tau^* = 20$ for the non-linear reweighting function to optimize the accuracy-robustness trade-off for DBP methods.

\input{tables/tau}

\textbf{Ablation study on sampling methods.} 
\emph{DiffPure} originally used an adjoint method to efficiently compute the gradients of the system, but~\citet{lee2023robust} and \citet{chen2024robust} suggest to replace adjoint solver with sdeint solver for the purpose of computing full gradients more accurately~\citep{li2020scalable,kidger2021neuralsde}.
Therefore, we investigate whether using different sampling methods affect the performance of DBP methods (here we use \emph{DiffPure} as the baseline method).
We further compare the results with DDIM sampling method \citep{song2021denoising}, which is a faster sampling method than DDPM \citep{ho2020denoising}.
From Table \ref{table:sampling methods}, DDPM achieves the best accuracy-robustness trade-off among the three sampling methods, and thus we select DDPM as the sampling method for all baseline methods.
\input{tables/Inference_Time_CIFAR10}

\textbf{Ablation study on score norms.} We investigate the effect of using single score norm (i.e., $\left\|\nabla_\rvx \log p_t(\rvx) \right\|$) for SSNI in Appendix \ref{A: score norm}. We find that although single score norm can notably improve the standard accuracy, it reduces robust accuracy. This might be attributed to the fact that single score norm is sensitive to the purification noise levels.

\textbf{Ablation study on the bias term $b$.} 
We investigate how the bias term $b$ in reweighting functions affects the performance of SSNI in Appendix \ref{A: bias}. We find that the selection of bias term will not significantly impact the performance of our framework under CIFAR-10 and ImageNet-1K. Note that when bias increases, there is a general observation that the clean accuracy drops and the robust accuracy increases. This perfectly aligns with the understanding of optimal noise level selections in existing DBP methods, where a large noise level would lead to a drop in both clean and robust accuracy and a small noise level cannot remove the adversarial perturbation effectively.

\textbf{Ablation study on model architectures.}
We also investigate how the choice of model architecture affects the performance of our method in Appendix~\ref{A: architecture}. Specifically, we evaluate \emph{SSNI-N} on a Swin-Transformer \citep{liu2021swin} under the PGD+EOT $\ell_{\infty}(\epsilon = 8/255)$ on \emph{CIFAR-10}. We find that the improvements brought by \emph{SSNI-N} are consistent with those observed on CNN-based models. In particular, \emph{SSNI-N} enhances both standard and robust accuracy across all evaluated DBP baselines. Notably, the relative improvement on robust accuracy is more significant for transformer-based classifiers. This observation suggests that \emph{SSNI-N} can effectively complement the inherent robustness of transformer models and generalizes well across different architectures.

\subsection{Compute Resource}
\label{sec: inference time}
The inference time (in seconds) for incorporating SSNI modules into existing DBP methods on CIFAR-10 and ImageNet-1K are reported in Tables \ref{table: inference time cifar10}. 
The inference time is measured as the time it takes for a single test image to complete the purification process.
Specifically, SSNI is approximately 0.5 seconds slower than baseline methods on CIFAR-10 and 5 seconds slower than baseline methods on ImageNet-1K.
Thus, compared with DBP baseline methods, this reweighting process is \emph{lightweight}, ensuring that SSNI is computationally feasible and can be applied in practice with minimal overhead.
We implemented our code on Python version 3.8, CUDA version 12.2.0, and PyTorch version 2.0.1 with Slurm Workload Manager. We conduct each of the experiments on up to 4 $\times$ NVIDIA A100 GPUs (see \url{https://github.com/tmlr-group/SSNI}).

\section{Limitation}

\textbf{Maximum level $t^{\rm S}$ for EPS.}
We use EPS to replace the single reference noise level with an integrated approach.
Still, the maximum level $t^{\rm S}$ defining the upper bound of the expectation range is shared across samples, which may not be the optimal choice though.
Refining sample-sensitive maximum levels could potentially improve the adaptability and robustness of EPS in different scenarios, and we leave it as future work.

\textbf{The design of reweighting functions.} 
The proposed reweighting functions (i.e., the linear and non-linear reweighting functions) may not be the optimal ones for SSNI. 
Further efforts are needed to explore data-driven or learning-based strategies to discover more effective function forms.
Meanwhile, designing an effective reweighting function is an open question, and we leave it as future work.

\textbf{Extra computational cost.} 
The integration of an extra reweighting process will inevitably bring some extra cost. Further efforts are needed to optimize the trade-off between accuracy gains and computational overhead.
Luckily, we find that this reweighting process is \emph{lightweight}, making SSNI computationally feasible compared to existing DBP methods (see Section \ref{sec: inference time}).

\section{Conclusion}
\label{Sec: Conclusion}
In this paper, we find that an optimal $t^*$ indeed could be different on a sample basis. Motivated by this finding, we propose a new framework called \emph{\textbf{S}ample-specific \textbf{S}core-aware \textbf{N}oise \textbf{I}njection} (SSNI). SSNI sample-wisely reweights $t^*$ for each sample based on its score norm, which generally injects less noise to clean samples and sufficient noise to adversarial samples, leading to a notable improvement in the accuracy-robustness trade-off. We hope this simple yet effective framework could open up a new perspective in DBP methods and lay the groundwork for future methods that account for sample-specific noise injections.

\section*{Acknowledgements}
JCZ is supported by the Melbourne Research Scholarship and would like to thank Shuhai Zhang and Huanran Chen for productive discussions.
ZSY is supported by the Australian Research Council (ARC) with grant number DE240101089.
FL is supported by the ARC with grant number DE240101089, LP240100101, DP230101540 and the NSF\&CSIRO Responsible AI program with grant number 2303037.
This research was supported by The University of Melbourne’s Research Computing Services and the Petascale Campus Initiative. Chaowei Xiao is supported by the Schmidt Sciences AI2050 fellow program.

\section*{Impact Statement}
This study on adversarial defense mechanisms raises important ethical considerations that we have carefully addressed.
We have taken steps to ensure our adversarial defense method is fair.
We use widely accepted public benchmark datasets to ensure comparability of our results.
Our evaluation encompasses a wide range of attack types and strengths to provide a comprehensive assessment of our defense mechanism. 
We have also carefully considered the broader impacts of our work.
The proposed defense algorithm contributes to the development of more robust machine learning models, potentially improving the reliability of AI systems in various applications.
We will actively engage with the research community to promote responsible development and use of adversarial defenses.

\bibliography{example_paper}
\bibliographystyle{icml2025}

\newpage
\appendix
\onecolumn

\section{Detailed Related Work}
\label{A: related work}

\textbf{Adversarial attack.} \emph{Adversarial examples} (AEs) have emerged as a critical security concern in the development of AI systems in recent years~\citep{szegedy2014intriguing, goodfellow2015explaining}.
These examples are typically generated by introducing imperceptible perturbations to clean inputs, which can cause a classifier to make incorrect predictions with high confidence. 
The methods used to generate such examples are known as \emph{adversarial attacks}. 
One of the earliest attack methods, the \emph{fast gradient sign method} (FGSM), which perturbs clean data in the direction of the gradient of the loss function~\citep{goodfellow2015explaining}. 
Building on this idea, \citet{madry2018towards} propose the \emph{projected gradient descent} (PGD) attack, which applies iterative gradient-based updates with random initialization.
\emph{AutoAttack} (AA)~\citep{croce2020reliable} integrate multiple attack strategies into a single ensemble, making it a widely adopted benchmark for evaluating adversarial robustness.
To address defenses that apply randomized input transformations, \citet{anish2018synthesizing} introduce the \emph{expectation over transformation} (EOT) framework for computing more accurate gradients. 
Additionally, \citet{athalye2018obfuscated} propose the \emph{backward pass differentiable approximation} (BPDA), which approximates gradients using identity mappings to circumvent gradient obfuscation defenses.
According to \citet{lee2023robust}, the combination of PGD and EOT, i.e., PGD+EOT, is currently considered the most effective attack strategy against DBP methods.
More recently, adversarial attacks that specifically designed for DBP methods are proposed.
For example, \citet{xue2024diffusion} propose \emph{diffusion-based projected gradient descent} (Diff-PGD), which leverages an off-the-shelf diffusion model to guide perturbation optimization, enabling the generation of more stealthy AEs. \citet{chen2023diffusion} propose \emph{DiffAttack}, which crafts perturbations in the latent space of diffusion models, rather than directly in pixel space, enabling the generation of human-insensitive yet semantically meaningful AEs through content-preserving structures.

\textbf{Adversarial defense.} To counter the threat posed by adversarial attacks, numerous defense strategies have been developed, including \emph{adversarial detection} (AD), \emph{adversarial training} (AT), and \emph{adversarial purification} (AP). 
\emph{(1) AD:} AD is the most lightweight approach to defending against adversarial attacks is to detect and remove AEs from the input data.
Earlier studies typically trained detectors tailored to specific classifiers or attack types, often overlooking the underlying data distribution, which limits their generalization to unseen attacks \citep{ma2018characterizing, lee2018a, raghuram2021a, pang2022two}.
Recently, \emph{statistical adversarial data detection} has attracted growing attention for its ability to address this limitation.
For instance, \citet{gao2021maximum} show that the \emph{maximum mean discrepancy} (MMD) is sensitive to adversarial perturbations, and use distributional discrepancies between AEs and CEs to effectively identify and filter out AEs, even under previously unseen attacks.
Building on this insight, \citet{zhang2023detecting} introduce a novel statistic called the \emph{expected perturbation score} (EPS), which quantifies the average score of a sample after applying multiple perturbations.
They then propose an EPS-based variant of MMD to capture the distributional differences between clean and adversarial examples more effectively. 
\emph{(2) AT:} Vanilla AT~\citep{madry2018towards} directly generates adversarial examples during training, encouraging the model to learn their underlying distribution.
Beyond vanilla AT, various extensions have been proposed to improve its effectiveness.
For instance, \citet{zhang2019theoretically} introduce a surrogate loss optimized based on theoretical robustness bounds.
Similarly, \citet{wang2020improving} examine the role of misclassified examples in shaping model robustness and enhance adversarial risk via regularization techniques.
From a reweighting perspective, \citet{zhang2021geometry} propose \emph{geometry-aware instance-reweighted AT} (GAIRAT), which adjusts sample weights based on their distance to the decision boundary.
Building on this, \citet{wang2021probabilistic} utilize probabilistic margins to reweight AEs in a continuous and path-independent manner.
More recently, \citet{zhang2024improving} suggest pixel-wise reweighting of AEs to explicitly direct attention toward critical image regions.
\emph{(3) AP: } AP typically employs generative models to transform AEs back into their clean counterparts before classification~\citep{liao2018defense, samangouei2018defensegan, song2018pixeldefend, naseer2020a, zhang2025ddad}.
Within this context, \emph{diffusion-based purification} (DBP) methods have emerged as a promising framework, exploiting the inherent denoising nature of diffusion models to filter out adversarial noise \citep{nie2022diffusion, wang2022guided, xiao2023densepure, lee2023robust}. 
DBP works with \emph{pre-trained} diffusion model by first corrupting an adversarial input through the forward process and then iteratively reconstructing it via the reverse process.
This projects the input back onto the clean data manifold while stripping away adversarial artifacts.
\citet{wang2022guided} introduce input guidance during the reverse diffusion process to ensure the purified outputs stay close to the inputs. 
\citet{lee2023robust} propose a fine-tuned gradual noise scheduling for multi-step purifications. 
\citet{bai2024diffusion} improve the reverse diffusion process by incorporating contrastive objectives.

\newpage
\section{Justification of Section~\ref{sec: unifying SSNI and Sample-shared DBP}}
\label{A: justification}

\textbf{Sampled-shared DBP is a special case of SSNI}.
\begin{proof}
    Let $\Phi_{\rm SH}$ be a sample-shared purification operator with constant noise level $t^*$.
    We can express $\Phi_{\rm SS}$ as a sample-specific purification operator $\Phi_{\rm SI}$ by defining the reweighting function $f$ as
    \begin{equation*}
        f(z, t^*) = t^* \quad \forall z \in \sR, t \in [0, T_{\rm max}].
    \end{equation*}
    Then, for any $\rvx \in \gX$, we have
    \begin{equation*}
        \begin{aligned}
            \Phi_{\rm SI}(\rvx) & = R(\rvx_{T(\rvx)}) \\
                                & = R(\rvx_{f(|| s_{\theta}(\rvx, t^*)||, t^*)}) \\
                                & = R(\rvx_{t^*}) \\
                                & = \Phi_{\rm SH}(\rvx).
        \end{aligned}
    \end{equation*}
\end{proof}

\textbf{SSNI has a Higher Purification Capacity}.

\noindent Statement 1~[Comparison of Purification Range]: For any $\rvx \in \gX$, we have $\Omega_{\rm SH} (\rvx)\subseteq \Omega_{\rm SI} (\rvx)$.
\begin{proof}
    Let $\rvy \in \Omega_{\rm SH}(\rvx)$. 
    Then $\exists t^* \in [0, T_{\rm max}]$ such that $\rvy = R(\rvx_{t^*})$.
    Define $f(z, t^*) = t^*$ for all $z$ and $t^*$, then $t(\rvx) = f(|| s_{\theta}(\rvx, t^*) ||, t^*) = t^*$.
    Therefore, $\rvy = R(\rvx_{t^*}) = R(\rvx_{t(\rvx)}) \in \Phi_{\rm SI}(\rvx)$.
    This completes the proof of $\Omega_{\rm SH} (\rvx)\subseteq \Omega_{\rm SI} (\rvx)$.
\end{proof}

\noindent Statement 2~[Strict Inclusion]: There exists $\rvx \in \gX$, we have $\Omega_{\rm SH} (\rvx)\subsetneq \Omega_{\rm SI} (\rvx)$.
\begin{proof}
    Consider a non-constant score function $s_{\theta}(\rvx, t)$ and a non-trivial reweighting function $f$. 
    We can choose $\rvx$ such that $t(\rvx) \neq t^*$ for any fixed $t^*$.
    Then $R(\rvx_{t(\rvx)}) \in \Phi_{\rm SI}(\rvx)$ but $R(\rvx_{t(\rvx)}) \neq \Phi_{\rm SH}$.
    This completes the proof of $\Omega_{\rm SH} (\rvx)\subsetneq \Omega_{\rm SI} (\rvx)$.
\end{proof}

\section{Defense Methods Configurations}
\label{Defense Methods Configurations}
For all chosen DBP methods, we utilize surrogate process to obtain gradients of the defense system during white-box adaptive attack, but we directly compute the full gradients during defense evaluation.
Furthermore, we consistently apply DDPM sampling method to the selected DBP methods, which means we replace the numeric SDE solver (\textbf{sdeint}) with DDPM sampling method in DiffPure~\citep{nie2022diffusion} and GDMP~\citep{wang2022guided}.
The reason is that the SDE solver does not support sample-specific timestep input.
For DDPM sampling, we can easily manipulate sample-specific timestep input by using matrix operation.
\subsection{DiffPure}
Existing DBP methods generally follow the algorithm of DiffPure~\citep{nie2022diffusion}. DiffPure conducts evaluation on AutoAttack~\citep{croce2020reliable} and BPDA+EOT adaptive attack~\citep{athalye2018obfuscated} to measure model robustness.
DiffPure chooses optimal $t^* = 100$ and $t^* = 75$ on CIFAR-10 against threat models $\ell_\infty (\epsilon = 8/255)$ and $\ell_2 (\epsilon = 0.5)$, respectively.
It also tests on high-resolution dataset like ImageNet-1K with $t^* = 150$ against threat models $\ell_\infty (\epsilon = 4/255)$.
Calculating exact full gradients of the defense system of DiffPure is impossible since one attack iteration requires 100 function calls (with $t^* = 100$ and a step size of 1).
DiffPure originally uses numerical SDE solver for calculating gradients.
However, the adjoint method is insufficient to measure the model robustness since it relies on the performance of the underlying SDE solver~\citep{zhuang2020adaptive,lee2023robust,chen2024robust}.
Therefore, we apply surrogate process to efficiently compute gradients of direct back-propagation in our evaluation.
To overcome memory constraint issue, we align the step size settings of denoising process in adversarial attack to 5 with the evaluation settings in~\citep{lee2023robust} and keep the timestep $t^*$ consistent with DiffPure. For ImageNet-1K evaluation, we can only afford a maximum of 10 function calls for one attack iteration. 

\subsection{GDMP}
GDMP basically follows the purification algorithm proposed in DiffPure~\citep{nie2022diffusion,wang2022guided}, but their method further introduces a novel guidance and use multiple purification steps sequentially.
GDMP proposes to use gradients of a distance between an initial input and a target being processed to preserve semantic information, shown in \eqref{eq: gdmp}.
\begin{equation}
    \rvx_{t-1} \sim \mathcal{N}(\boldsymbol{\mu}_\theta - s\boldsymbol{\Sigma}_{\theta} \nabla_{\rvx^{t}} \mathcal{D}(\rvx^{t}, \rvx^t_{\text{adv}}), \boldsymbol{\Sigma}_{\theta}),
    \label{eq: gdmp}
\end{equation}
Given a DDPM $(\mu_{\phi}(\rvx_t), \Sigma_{\theta}(\rvx_t))$, a gradient scale of guidance $s$. 
$\rvx_{t}$ is the data being purified, and $\rvx^t_{\text{adv}}$ is the adversarial example at $t$.
Also, GDMP empirically finds that multiple purification steps can improve the robustness.
In the original evaluation of GDMP, the defense against the PGD attack consists of four purification steps, with each step including 36 forward steps and 36 denoising steps.
For BPDA+EOT adaptive attack, GDMP uses two purification steps, each consisting of 50 forward steps and 50 denoising steps.

\citet{lee2023robust} evaluated GDMP with three types of guidance and concluded that No-Guidance provides the best robustness performance when using the surrogate process to compute the full gradient through direct backpropagation.
In our evaluation, we incorporate the surrogate process with No-Guidance to evaluate GDMP.
Since it is impossible to calculate the gradients of the full defense system, we use a surrogate process consisting of same number of purification steps but with larger step size in the attack (with 6 denoising steps and 10 denoising steps for PGD+EOT and BPDA+EOT attack, respectively). Notably, GDMP only uses one purification run with 45 forward steps to evaluate on ImageNet-1K, which we keep consistent with this setting.

\subsection{GNS}
\citet{lee2023robust} emphasizes the importance of selecting optimal hyperparameters in DBP methods for achieving better robust performance.
Hence, \citet{lee2023robust} proposed Gradual Noise-Scheduling (GNS) for multi-step purification, which is based on the idea of choosing the best hyperparameters for multiple purification steps. It is basically the same architecture as GDMP (no guidance), but with different purification steps, forward steps and denoising steps.
Specifically, GNS sets different forward and reverse diffusion steps and gradually increases the noise level at each subsequent purification step.
We just keep the same hyperparameter settings and also use an ensemble of 10 runs to evaluate the method.

\section{Surrogate Process of Gradient Computation}
\label{Surrogate Process of Gradient Computation}
The surrogate process is an efficient approach for computing approximate gradients through backpropagation, as proposed by \citep{lee2023robust}.
White-box adaptive attacks, such as PGD+EOT, involve an iterative optimization process that requires computing the exact full gradients of the entire system, result in high memory usage and increased computational time.
DBP methods often include a diffusion model as an add-on purifier, which the model requires extensive function calls during reverse generative process.
Hence, it is hard to compute the exact full gradient of DBP systems efficiently.
The surrogate process takes advantage of the fact that, given a fixed total amount of noise, we can denoise it using fewer denoising steps \citep{song2021denoising}, but the gradients obtained from the surrogate process slightly differ from the exact gradients.
Instead of using the full denoising steps, we approximate the original denoising process with fewer function calls, which allows us to compute gradients by directly back-propagating through the forward and reverse processes.

There are other gradient computation methods such as adjoint method in DiffPure \citep{li2020scalable,nie2022diffusion}.
It leverages an underlying numerical SDE solver to solve the reverse-time SDE.
The adjoint method can theoretically compute exact full gradient, but in practice, it relies on the performance of the numerical solver, which is insufficient to measure the model robustness \citep{zhuang2020adaptive,lee2023robust,chen2024robust}.
\citet{lee2023robust} conducted a comprehensive evaluation of both gradient computation methods and concluded that utilizing the surrogate process for gradient computation poses a greater threat to model robustness.
Hence, we use gradients obtained from a surrogate process in all our experiments.

\newpage
\section{Adaptive White-box PGD+EOT Attack for SSNI}
\label{A: adaptive pgd+eot}
\citet{madry2018towards} proposed Projected Gradient Descent (PGD), which is a strong iterative adversarial attack. Combining PGD with Expectation Over Transformation (EOT) \citep{anish2018synthesizing} has become a powerful adaptive white-box attack against AP methods. In our defense system, we add a reweighting module to existing DBP methods, so this information must be included during a white-box attack. We decide to place the reweighting process outside the EOT loops because EOT provides more samples with different random transformations.
It effectively reduces the randomness of gradient computation but this process does not greatly affect the EPS value of the data. 
In addition, this approach further reduces the computational cost.
\input{algorithm/PGD+EOT-adaptive}

\section{Adaptive White-box BPDA+EOT Attack for SSNI}
\label{BPDA+EOT Adaptive Attack}
\citet{athalye2018obfuscated} proposed Backward Pass Differentiable Approximation (BPDA), which is a popular adaptive attack to DBP defense methods.
When a defense system contains non-differentiable components where gradients cannot be directly obtained, the BPDA method substitutes these non-differentiable operations with a differentiable approximation during backpropagation in order to compute the gradients.
It typically utilizes an identity function $f(\rvx) = \rvx$ as the differentiable approximation function.
In specific, computing the exact full gradients via direct backpropagation through a diffusion model is time-consuming and memory-intensive, so BPDA attack assumes the purification process is an identity mapping.
This implies that we ignore the impact of the purification on the gradients and directly treat the gradient with respect to the purified data $\hat{\rvx}_0$ as the gradient with respect to the input $\rvx_0$.
\input{algorithm/BPDA+EOT-adaptive}

\section{Performance Evaluation of SSNI-L}
\label{Sec: Performance Evaluation of SSNI-L}
We also incorporate \emph{SSNI-L} reweighting framework with existing DBP methods for accuracy-robustness evaluation.
In Table~\ref{table:cifar10 SSNI-L SUB}, we report the results against PGD+EOT $\ell_{\infty}(\epsilon=8/255)$ and $\ell_{2}(\epsilon=0.5)$ threat models on CIFAR-10, respectively.
We can see that \emph{SSNI-L} can still support DBP methods to better trade-off between standard accuracy and robust accuracy.
Also, we report the results against BPDA+EOT $\ell_{\infty}(\epsilon=8/255)$ threat model on CIFAR-10 in Table~\ref{table:bpda2810}.
Overall, \emph{SSNI-L} slightly decreases the robustness of DBP methods against PGD+EOT attack and maintain the robustness of DBP methods against BPDA+EOT attack, but there is a notable improvement in standard accuracy.
\newpage
\input{tables/CIFAR10-SUB-SSNI-L}
\input{tables/BPDA-2810_SSNI-L}

\section{Ablation Study on Score Norm}
\label{A: score norm}
We further provide experiments with the single score norm instead of the EPS norm. \citet{yoon2021adversarial} shows score norm $\nabla_{\rvx}{\log p_t(\rvx)}$ is a valid measurement for adversarial detection. Incorporating single score norm with our SSNI-N framework, it still achieves notable improvement on standard accuracy, but the robustness drops. The single score norm is highly sensitive to noise levels, which makes it insufficient to completely distinguish between natural and adversarial examples.
\input{tables/CIFAR10-SUB-Score}

\section{Ablation Study on Bias Term}
\label{A: bias}
\vspace{-1em}
\begin{table*}[htbp]
\centering
\caption{Ablation study on the bias term $b$. We report the standard and robust accuracy of DBP methods against adaptive white-box PGD+EOT on \emph{CIFAR-10}. WideResNet-28-10 and WideResNet-70-16 are used as classifiers. We report mean and standard deviation over three runs. We show the most successful defense in \textbf{bold}.}
\vspace{2pt}
\label{table: bias term}
\footnotesize
\centering
\begin{tabular}{lcccccccc}
\toprule
\multicolumn{9}{c}{PGD+EOT $\ell_\infty~(\epsilon = 8/255)$}\\
\midrule
& Bias & 0 & 5 & 10 & 15 & 20 & 25 & 30 \\
\midrule
& \multicolumn{8}{c}{\cellcolor{lg}{WRN-28-10}} \\
\cmidrule{2-9}
\multirow{5}{*}{\rotatebox[origin=c]{90}{\emph{SSNI-N}}} & Standard & \bf{94.34$\pm$1.43} & 93.95$\pm$1.17 & 93.17$\pm$1.05 & 92.38$\pm$1.29 & 92.38$\pm$1.29 & 92.10$\pm$1.03 & 92.97$\pm$0.37\\
& Robust & 55.27$\pm$0.75 & 57.23$\pm$1.62 & 57.03$\pm$0.54 & 57.64$\pm$0.89 & 57.64$\pm$0.89 & 58.24$\pm$1.22 & \bf{59.18$\pm$1.65}\\
\cmidrule{2-9}
& \multicolumn{8}{c}{\cellcolor{lg}{WRN-70-16}} \\
\cmidrule{2-9}
& Standard & 94.34$\pm$0.45 & 93.82$\pm$1.56 & \bf{94.73$\pm$1.34} & 92.79$\pm$0.89 & 92.79$\pm$0.89 & 92.58$\pm$0.67 & 92.77$\pm$0.94\\
& Robust & 56.45$\pm$1.22 & 57.03$\pm$0.78 & 58.10$\pm$0.24 & 58.79$\pm$1.48 & \bf{59.57$\pm$1.19} & 58.59$\pm$0.52 & 58.59$\pm$0.52\\
\midrule
& \multicolumn{8}{c}{\cellcolor{lg}{WRN-28-10}} \\
\cmidrule{2-9}
\multirow{5}{*}{\rotatebox[origin=c]{90}{\emph{SSNI-L}}} & Standard & 92.91$\pm$0.55 & 92.97$\pm$0.42 & \bf{93.01$\pm$0.78} & 92.64$\pm$0.28 & 92.53$\pm$0.84 & 91.69$\pm$1.02 & 91.45$\pm$0.88 \\
& Robust & 46.29$\pm$0.46 & 46.35$\pm$0.72 & 46.28$\pm$0.35 & 46.75$\pm$0.63 & 47.11$\pm$0.92 & 47.05$\pm$0.50 & \bf{47.23$\pm$0.76} \\
\cmidrule{2-9}
& \multicolumn{8}{c}{\cellcolor{lg}{WRN-70-16}} \\
\cmidrule{2-9}
& Standard & 93.79$\pm$0.53 & \bf{93.82$\pm$0.49} & 93.77$\pm$0.64 & 92.85$\pm$0.82 & 92.32$\pm$0.90 & 92.08$\pm$0.40 & 91.89$\pm$0.52\\
& Robust & 49.34$\pm$0.85 & 49.94$\pm$0.33 & 48.83$\pm$0.92 & 50.72$\pm$0.77 & 50.80$\pm$0.94 & \bf{51.25$\pm$0.66} & 51.05$\pm$1.11 \\
\bottomrule
\end{tabular}
\end{table*}

\begin{table*}[!htbp]
\centering
\caption{Ablation study on the bias term $b$. We report standard and robust accuracy of DBP methods against adaptive white-box PGD+EOT on \emph{ImageNet-1K}. ResNet-50 is used as the classifier. We report mean and standard deviation over three runs. We show the most successful defense in \textbf{bold}.}
\vspace{2pt}
\label{table: bias term imagenet}
\footnotesize
\centering
\begin{tabular}{lcccccccc}
\toprule
\multicolumn{9}{c}{PGD+EOT $\ell_\infty~(\epsilon = 4/255)$}\\
\midrule
& Bias & 0 & 25 & 50 & 75 & 100 & 125 & 150 \\
\midrule
\multirow{3}{*}{\rotatebox[origin=c]{90}{\emph{SSNI-N}}} & \multicolumn{8}{c}{\cellcolor{lg}{RN-50}} \\
\cmidrule{2-9}
& Standard & 71.68$\pm$1.12 & 71.73$\pm$1.49 & \bf{71.96$\pm$0.13} & 68.80$\pm$0.74 & 68.41$\pm$0.59 & 67.63$\pm$1.08 & 66.45$\pm$1.53\\
& Robust & 39.33$\pm$0.34 & 40.28$\pm$0.28 & \bf{43.88$\pm$0.22} & 41.45$\pm$0.38 & 43.02$\pm$0.41 & 40.87$\pm$0.92 & 40.05$\pm$0.67 \\
\bottomrule
\end{tabular}
\end{table*}

\section{Ablation Study on Model Architecture}
\label{A: architecture}
\begin{table}[!htbp]
\centering
\caption{Standard and robust accuracy (\%) against adaptive white-box PGD+EOT $\ell_\infty (\epsilon = 8/255)$ attack on \emph{CIFAR-10}. The target classifier is a Swin-Transformer. We report mean and standard deviation over three runs. We show the most successful defense in \textbf{bold}.}
\vspace{2pt}
\footnotesize
\begin{tabular}{clcc}
\toprule
\multicolumn{4}{c}{PGD+EOT $\ell_\infty~(\epsilon = 8/255)$}\\
\midrule
& DBP Method & Standard & Robust \\
\midrule
\multirow{7}{*}{\rotatebox[origin=c]{90}{Swin-T}} & \citet{nie2022diffusion} & 88.48$\pm$0.49 & 37.11$\pm$0.83\\
& \cellcolor{lg}{+ \emph{SSNI-N}} & \cellcolor{lg}{\bf{88.67$\pm$0.55}} & \cellcolor{lg}{\bf{39.65$\pm$0.47}} \\
\cmidrule{2-4}
& \citet{wang2022guided} & 88.09$\pm$0.72 & 27.34$\pm$0.18 \\
& \cellcolor{lg}{+ \emph{SSNI-N}} & \cellcolor{lg}{\bf{89.45$\pm$0.34}} & \cellcolor{lg}{\bf{28.71$\pm$0.27}} \\
\cmidrule{2-4}
& \citet{lee2023robust} & 88.67$\pm$0.20 & 52.54$\pm$0.14 \\
& \cellcolor{lg}{+ \emph{SSNI-N}} & \cellcolor{lg}{\bf{91.21$\pm$0.11}} & \cellcolor{lg}{\bf{54.10$\pm$0.22}} \\
\bottomrule
\end{tabular}
\end{table}

\newpage
\section{Additional Justification of the Relationship Between Score Norms and Noise Levels}
\label{A: proof}

Extending the empirical findings from Section~\ref{Sec: motivation}, we further justify how score norms of different samples correlate with the noise level $t^{*}$ in diffusion models.
With Proposition~\ref{prop}, we find that the score norm varies with $t$, and the variation exceeds a threshold $\epsilon > 0$ when the time difference $|t_2 - t_1|$ is sufficiently large.
Thus, for two noisy samples $\vx_1$ and $\vx_2$ with {\it different score norms}, the monotonic increase of $\{\beta_t \}_{t \in [0, T]}$ with $t$ implies that different score norms correspond to different noise levels, i.e., $t_1 \neq t_2$. 
It implies that, for samples with different score norms, it is likely that they have undergone {\it different noise injection steps}.
\footnote{However, we note that this analysis only focuses on the property of diffusion models, disregarding the potential interactions between Gaussian and adversarial noise (which is often difficult to analyze directly). In this sense, it does {\it not} formally motivate the proposal of SSNI and is {\it not} considered a core contribution of this study.
We include this part of the content for completeness.
}.

\begin{definition}[Marginal Probability Density]\label{def: prob_density}
    Let $\gX$ be the image sample space, $p(\rvx_0)$ be the natural data distribution over $\gX$,
     and the diffusion process kernel defined in \eqref{eq: forward}. 
    The marginal probability density $p_t: \gX \to \sR_{+}$ at time $0 \leq t \leq T$ can be expressed as:
    \begin{equation}\label{eq: prob_density}
        p_t(\rvx) = \int_{\gX} q(\rvx | \rvx_0) p(\rvx_0) d \rvx_0,
    \end{equation}
    where $q(\rvx | \rvx_0)$ describes how a natural sample $\rvx_0$ evolves under the forward process to $\rvx$ at time $t$.    
\end{definition}

Before proceeding with the proof of Lemma~\ref{lem3} and Proposition~\ref{prop}, we start by presenting two lemmas to facilitate the proof.
\begin{lemma}\label{lem1}
    Let $p_t (\rvx)$ denote the marginal probability density of $\rvx$ at time $t$.
    For any $\rvx \in \gX$ and $0 \leq t \leq T$, the score function $\nabla_{\rvx} \log p_t (\rvx)$ at time $t$ can be expressed as:
    \begin{equation*}
        \nabla_{\rvx} \log p_t (\rvx) = \E_{p(\rvx_0 | \rvx)} \left[ \nabla_{\rvx} \log q(\rvx | \rvx_0) \right],
    \end{equation*}
    where $p(\rvx_0 | \rvx)$ is the posterior given by Bayes' Rule:
    \begin{equation*}
        p(\rvx_0 | \rvx) = \frac{ q(\rvx | \rvx_0) p(\rvx_0) }{ \int q(\rvx | \rvx_0) p(\rvx_0) d \rvx_0 }.
    \end{equation*}
\end{lemma}
\begin{proof}
    \begin{equation}
        \begin{aligned}
            \nabla_{\rvx} \log p_t (\rvx) & = \nabla_{\rvx} \log \int q(\rvx | \rvx_0) p(\rvx_0) d \rvx_0  & \text{(by Def.~\ref{def: prob_density})} \\
                                          & = \frac{\nabla_{\rvx} \int q(\rvx | \rvx_0) p(\rvx_0) d \rvx_0}{\int q(\rvx | \rvx_0) p(\rvx_0) d\rvx_0} & \text{(chain rule)} \\
                                          & = \frac{\int \nabla_{\rvx} q(\rvx | \rvx_0) p(\rvx_0) d \rvx_0}{\int q(\rvx | \rvx_0) p(\rvx_0) d\rvx_0} & \text{(Leibniz integral rule)} \\
                                          & = \frac{\int \frac{ \nabla_{\rvx} q(\rvx | \rvx_0)}{q(\rvx | \rvx_0)} q(\rvx | \rvx_0) p(\rvx_0) d \rvx_0}{\int q(\rvx | \rvx_0) p(\rvx_0) d\rvx_0} & \text{(manipulate } q(\rvx | \rvx_0) \text{)}& \\
                                          & = \frac{\int (\nabla_{\rvx} \log q(\rvx | \rvx_0)) q(\rvx | \rvx_0) p(\rvx_0) d \rvx_0}{\int q(\rvx | \rvx_0) p(\rvx_0) d\rvx_0} & \text{(chain rule)} \\
                                          & = \int (\nabla_{\rvx} \log q(\rvx | \rvx_0)) p(\rvx_0  | \rvx) d \rvx_0 & \text{(} p(\rvx_0  | \rvx) \triangleq \frac{ q(\rvx | \rvx_0) p(\rvx_0) }{ \int q(\rvx | \rvx_0) p(\rvx_0) d \rvx_0 } \text{)}
        \end{aligned}
    \end{equation}
    This ends up with the expectation over the posterior $p(\rvx_0 | \rvx)$ as $\E_{p(\rvx_0 | \rvx)} \left[ \nabla_{\rvx} \log q(\rvx | \rvx_0) \right]$.
\end{proof}
Lemma~\ref{lem1} establishes a relationship between the marginal density and the forward process, where we can express the behavior of $\nabla_{\rvx} \log p_t (\rvx)$ in terms of the well-defined $q(\rvx | \rvx_0)$. 

\begin{lemma}\label{lem2}
    Consider the forward process $q(\rvx_t | \rvx_0)$ is defined with noise schedule $\{\beta_t\}_{t \in [0, T]}$, where $\beta_t \in (0, 1)$ for all $0 \leq t \leq T$.
    Then, for any $\rvx \in \gX$ and $0 \leq t \leq T$, the score function $\nabla_{\rvx} \log p_t (\rvx)$ can be decomposed as:
    \begin{equation*}
        \nabla_{\rvx} \log p_t (\rvx) = -g(t)\rvx + \E_{p(\rvx_0|\rvx)} \left[ g(t)\sqrt{\bar{\alpha}_t}\rvx_0 \right],
    \end{equation*}
where $g(t) = 1/(1 - \bar{\alpha}_t)$ and $\bar{\alpha}_t = \prod_{i=1}^{t} (1 - \beta_i)$.
\end{lemma}
\begin{proof}
    Denote the score function $s_t(\rvx) = \nabla_{\rvx} \log p_t(\rvx) $ for convenience.
    By applying the results from Lemma~\ref{lem1}, we write:
    \begin{equation*}
        s_t(\rvx) = \E_{p(\rvx_0 | \rvx)} \left[ q(\rvx | \rvx_0) \right].
    \end{equation*}
    Provided that $q(\rvx | \rvx_0) = \gN \left( \rvx; \sqrt{\bar{\alpha}_t} \rvx_0, (1 - \bar{\alpha}_t) \rmI \right)$, we compute the gradient of the log-probability of the forward process as:
    \begin{equation*}
        \begin{aligned}
             \nabla_{\rvx} \log q (\rvx | \rvx_0) & = \nabla_{\rvx} \gN \left( \rvx; \sqrt{\bar{\alpha}_t} \rvx_0, (1 - \bar{\alpha}_t) \rmI \right) \\
                      & = \nabla_{\rvx} \left[ -\frac{1}{2(1 - \bar{\alpha}_t)} (\rvx - \sqrt{\bar{\alpha}_t} \rvx_0)^{\top} (\rvx - \sqrt{\bar{\alpha}_t} \rvx_0) + C \right] \\
                      & = - \frac{1}{1 - \bar{\alpha}_t} (\rvx - \sqrt{\bar{\alpha}_t} \rvx_0), 
        \end{aligned}
    \end{equation*}
    where $C$ is a constant term independent of $\rvx$.
    Substitute the result back into $s_t(\rvx)$, we have:
    \begin{equation*}
        \begin{aligned}
            s_t(\rvx) & = \E_{p(\rvx_0 | \rvx)} \left[ q(\rvx | \rvx_0) \right] \\
                      & = - \frac{1}{1 - \bar{\alpha}_t} \E_{p(\rvx_0 | \rvx)} \left[ \rvx - \sqrt{\bar{\alpha}_t} \rvx_0 \right] \\
                      & = - \frac{1}{1 - \bar{\alpha}_t} \rvx + \frac{\sqrt{\bar{\alpha}_t}}{1 - \bar{\alpha}_t} \E_{p(\rvx_0 | \rvx)} [\rvx_0] \\
                      & = -g(t) \rvx + \E_{p(\rvx_0 | \rvx)} \left[ g(t) \sqrt{\bar{\alpha}_t} \rvx_0 \right].
        \end{aligned}
    \end{equation*}
    The Lemma is completed by substituting $g(t) = \frac{1}{1 - \bar{\alpha}_t}$ into the last equation.
\end{proof}
Lemma~\ref{lem2} expresses the score function by the noisy data $\rvx$ with noise level $t$ and the clean data $\rvx_0$.
We next examine how the score norm evolves to further understand its behavior as time increases.

\begin{lemma}
    Suppose there exists a constant $K > 0$ such that for all $t \geq 0$ and all $\rvx_t \in \gX$, the expected norm of the clean data given $\rvx_t$ satisfies: $\E_{\rvx_0 \sim p(\rvx_0 | \rvx)}[ \left\| \rvx_0 \right\|] \leq K \left\| \rvx_t \right\|$.
    Then, there exist constants $0 < C < 1$ and $T_0 > 0$ such that for all $t \geq T_0$:
    \begin{equation*}
        \frac{\left\| \nabla_{\rvx} \log p_t(\rvx) \right\|}{\left\| \rvx \right\|} > C.
    \end{equation*}
    \label{lem3}
\end{lemma}
\begin{proof}
    Again, denote $s_t(\rvx) = \nabla_{\rvx} \log p_t(\rvx) $.
    From Lemma~\ref{lem2}, we have:
    \begin{equation*}
        s_t(\rvx) = -g(t) \rvx + \E_{p(\rvx_0 | \rvx)} \left[ g(t) \sqrt{\bar{\alpha}_t} \rvx_0 \right].
    \end{equation*}
    Applying the triangle inequality to it, we have:
    \begin{equation*}
        \begin{aligned}
            \left\| s_t(\rvx) \right\| & \geq g(t) \left\| \rvx \right\| - \left\| \E_{p(\rvx_0 | \rvx)} \left[ g(t) \sqrt{\bar{\alpha}_t} \rvx_0 \right] \right\|  \\
                                       & \geq g(t) \left\| \rvx \right\| - g(t) \sqrt{\bar{\alpha}_t} \E_{p(\rvx_0 | \rvx)} \left[ \left\| \rvx_0 \right\| \right],
        \end{aligned}
    \end{equation*}
    where the second inequality comes from applying Jensen's inequality: $f(\E[X]) \leq \E[f(X)]$ for convex function $f$ and random variable $X$, which leads to:
    \begin{equation*}
        \begin{aligned}
        \left\| \E_{p(\rvx_0 | \rvx)} \left[ g(t) \sqrt{\bar{\alpha}_t} \rvx_0 \right] \right\| & \leq \E_{p(\rvx_0 | \rvx)} \left[ \left\| g(t) \sqrt{\bar{\alpha}_t} \rvx_0  \right\| \right] \\
         & = g(t) \sqrt{\bar{\alpha}_t} \E_{p(\rvx_0 | \rvx)} \left[ \left\| \rvx_0 \right\| \right].
        \end{aligned}
    \end{equation*}
    As we have assumed the existence of $K > 0$ that makes $\E_{\rvx_0 \sim p(\rvx_0 | \rvx)}[ \left\| \rvx_0 \right\|] \leq K \left\| \rvx_t \right\|$ holds for all $\rvx \in \gX$, we have:
    \begin{equation*}
        \begin{aligned}
            \left\| s_t(\rvx) \right\| & \geq g(t) \left\| \rvx \right\| - g(t) \sqrt{\bar{\alpha}_t} K \left\| \rvx \right\| \\
            & = g(t) \left\| \rvx \right\| ( 1 - K \sqrt{\bar{\alpha}_t}).
        \end{aligned}
    \end{equation*}
    We then turn to look into the asymptotic behavior of $g(t)$.
    We have $\bar{\alpha}_t = \prod_{s=1}^{t}(1 - \beta_s)$. 
    Since $0 < 1 - \beta_s < 1$ for all $s$, and the sequence is decreasing (as $\beta_s$ is increasing), it is easy to check that
    \begin{equation*}
        \lim_{t \to \infty} \bar{\alpha}_t = 0
    \end{equation*}
    and 
    \begin{equation*}
        \lim_{t \to \infty} g(t) = \lim_{t \to \infty} \frac{1}{1-\bar{\alpha}_t} = 1.
    \end{equation*}
    This can then be formalized as,
    for any $\epsilon > 0$, there exists a $T_\epsilon$ such that for all $t > T_\epsilon$:
    \begin{equation*}
        |g(t) - 1| < \epsilon, \text{ and } \sqrt{\bar{\alpha}_t} < \epsilon.
    \end{equation*}
    Then, for $t > T_\epsilon$, we have
    \begin{equation*}
        \begin{aligned}
            \left\| s_t(\rvx) \right\| & \geq g(t) \left\| \rvx \right\| ( 1 - K \sqrt{\bar{\alpha}_t} ) \\
                                       & > (1 - \epsilon) \left\| \rvx \right\| (1 - K \epsilon), \\
                                       & = (1 - \epsilon - K \epsilon + K\epsilon^2) \left\| \rvx \right\| \\
                                       & = (1 - (K + 1) \epsilon + K \epsilon^2) \left\| \rvx \right\|.
        \end{aligned}
    \end{equation*}
    To establish the desired inequality $\left\| s_t(\rvx) \right\| \geq C \left\| \rvx \right\|$ for constants $C > 0$, we next investigate whether this quatratic inequality $K \epsilon^2 + (K + 1) \epsilon + (1 - C) > 0$ can be solved.
    Denote the discriminant $D = (K + 1)^2 - 4 K (1 - C)$, we have
    \begin{equation*}
        \begin{aligned}
            D & = (K + 1)^2 - 4 K (1 - C) \\
              & = K^2 - 2K + 1 + 4KC \\
              & = (K - 1)^2 + 4KC \\
              & > 0 \qquad \text{(since } (K - 1)^2 \geq 0 \text{ and } K > 0 \text{).}
        \end{aligned}
    \end{equation*}
    Then, let $\epsilon_1 < \epsilon_2$ be two real roots of $K \epsilon^2 + (K + 1) \epsilon + (1 - C) = 0$, as the parabola opens upwards, i.e., $K > 0$, note that $\epsilon > 0$, we further need to ensure the smaller root $\epsilon = \frac{K + 1 - \sqrt{D}}{2K} > 0$, such that 
    \begin{equation*}
        \begin{aligned}
        K + 1     & > \sqrt{(K + 1)^2 - 4 K (1 - C)} \\
        (K + 1)^2 & > (K - 1)^2 + 4KC \\
               4K & > 4KC \\
                1 & > C.
        \end{aligned}
    \end{equation*}
    Putting them together, we have $\epsilon \in (0, \epsilon_1) \cup (\epsilon_2, \infty)$ that makes $\left\| s_t(\rvx) \right\| \geq C \left\| \rvx \right\|$ hold, given a constant $0 < C < 1$ in relation to $K$ and $\epsilon$.
    Setting $T_0 = T_\epsilon$ completes the proof.
\end{proof}

\begin{proposition}
    Consider the diffusion model satisfying all conditions as specified in Lemma~\ref{lem3}.
    Assume that there exist constants $K > 0$, such that $\beta_t \leq K$ for all $t \geq 0$.
    Additionally, suppose $\left\| \rvx \right\| \leq M$ for any $\rvx \in \gX$, for some $M > 0$.
    Then, for any $\epsilon$, there exists a constant $\Delta = 2\epsilon / (CK)$ such that for $t_1, t_2 \geq 0$, we have:
    \begin{equation*}
        | \left\| \nabla_{\rvx} \log p_{t_1}(\rvx)\right\| -  \left\|\nabla_{\rvx} \log p_{t_2}(\rvx) \right\| | > \epsilon, \; \text{ with $|t_1 - t_2| \geq \Delta$.}
    \end{equation*}
    \label{prop}
\end{proposition}
\begin{proof}
    Denote $s_t(\rvx) = \nabla_{\rvx} \log p_t(\rvx) $.
    Recall that we can express the score function as 
    \begin{equation*}
        \begin{aligned}
            s_t(\rvx) & = \E_{p(\rvx_0 | \rvx)} \left[ q(\rvx | \rvx_0) \right] \\
                      & = - \frac{1}{1 - \bar{\alpha}_t} \E_{p(\rvx_0 | \rvx)} \left[ \rvx - \sqrt{\bar{\alpha}_t} \rvx_0 \right] \\
                      & = - \frac{1}{1 - \bar{\alpha}_t} \rvx + \frac{\sqrt{\bar{\alpha}_t}}{1 - \bar{\alpha}_t} \E_{p(\rvx_0 | \rvx)} [\rvx_0]
        \end{aligned}
    \end{equation*}
    Let $g(t) = \left\| s_t(\rvx) \right\|^2$.
    We compute $\frac{\partial g(t)}{\partial t}$ using the chain rule and the product rule:
    \begin{equation*}
        \begin{aligned}
            \frac{\partial g(t)}{\partial t}& = 2 \left< s_t(\rvx), \frac{\partial s_t(\rvx)}{\partial t} \right> \\
                                          & = 2 \left< s_t(\rvx), \frac{\partial }{\partial t} \left( - \frac{1}{1 - \bar{\alpha}_t} \rvx + \frac{\sqrt{\bar{\alpha}_t}}{1 - \bar{\alpha}_t} \E_{p(\rvx_0 | \rvx)} [\rvx_0] \right) \right>
        \end{aligned}
    \end{equation*}
    Recall that $\bar{\alpha}_t = \prod_{i=1}^{t} (1 - \beta_i)$, we next derive the derivative of $\bar{\alpha}_t$.
    We consider the continuous approximation of $t$, where the product is approximated as an exponential of the integral
    \begin{equation*}
        \bar{\alpha}_t = \exp \left( \int_{0}^{t} \log ( 1 - \beta_i) \mathrm{d} i \right).
    \end{equation*}
    As $\beta_t$ is typically assumed to be small (which is an implicit common practice~\citep{ho2020denoising}), we further simplify $\log( 1 - \beta_i)$ with its first-order Taylor approximation, i.e., $\log( 1 - \beta_i) \approx - \beta_i$, which thus leads to the approximated $\bar{a}_t \approx \exp \left( - \int_0^t \beta_i \mathrm{d} i \right)$.
    This way we compute the derivative of $\bar{a}_t$ as
    \begin{equation*}
        \begin{aligned}
            \frac{\partial }{\partial t} \bar{a}_t & = \frac{\partial }{\partial t} \exp \left( - \int_0^t \beta_i \mathrm{d} i \right) \\
                & = \exp \left( - \int_0^t \beta_i \mathrm{d} i \right)  \frac{\partial }{\partial t} \left( - \int_0^t \beta_i \mathrm{d} i \right) \\
                & = \exp \left( - \int_0^t \beta_i \mathrm{d} i \right) (-\beta_t) \\
                & = - \beta_t \bar{\alpha}_t
        \end{aligned}
    \end{equation*}
    
    We can then compute the derivatives of the coefficients:
    \begin{equation*}
        \begin{aligned}
        \frac{\partial}{\partial t}\left(\frac{1}{1-\bar{\alpha}_t}\right) & = \frac{-1}{(1 - \bar{a}_t)^2 } \frac{\partial }{\partial t} \left( 1 - \bar{a}_t \right) \\
                               & = \frac{\beta_t \bar{\alpha}_t}{(1-\bar{\alpha}_t)^2}
        \end{aligned}
    \end{equation*}
    
    \begin{equation*}
        \begin{aligned}
            \frac{\partial}{\partial t}\left(\frac{\sqrt{\bar{\alpha}_t}}{1-\bar{\alpha}_t}\right) & = \frac{\frac{\partial}{\partial t} \sqrt{\bar{\alpha}_t} (1 - \bar{\alpha}_t) - \sqrt{\bar{\alpha}_t} \frac{\partial}{\partial t}( 1 - \bar{\alpha}_t) }{ (1 - \bar{\alpha}_t)^2 } \\
            & = \frac{ \left( \frac{1}{2} \bar{\alpha}_t^{-1/2} \frac{\partial \bar{\alpha}_t}{\partial t} \right) (1 - \bar{\alpha}_t) + \sqrt{\bar{\alpha}_t} \frac{\partial \bar{\alpha}_t}{\partial t} }{ (1 - \bar{\alpha}_t)^2 } \\
            & = \frac{ \left( \frac{1}{2} \bar{\alpha}_t^{-1/2} \left( -\beta_t \bar{\alpha}_t \right) \right) (1 - \bar{\alpha}_t) + \sqrt{\bar{\alpha}_t} \left( -\beta_t \bar{\alpha}_t \right) }{ (1 - \bar{\alpha}_t)^2 } \\
            & = \frac{ -\frac{1}{2} \beta_t \bar{\alpha}_t^{1/2} (1 - \bar{\alpha}_t) - \beta_t \bar{\alpha}_t^{3/2} }{ (1 - \bar{\alpha}_t)^2 } \\
            & = -\beta_t \bar{\alpha}_t^{1/2} \frac{ \frac{1}{2} (1 - \bar{\alpha}_t) + \bar{\alpha}_t }{ (1 - \bar{\alpha}_t)^2 } \\
            & = -\beta_t \sqrt{\bar{\alpha}_t} \frac{ 1 + \bar{\alpha}_t }{ 2(1 - \bar{\alpha}_t)^2 }
        \end{aligned}
    \end{equation*}
    Using the computed derivatives:
    \begin{equation*}
        \frac{\partial s_t(\rvx) }{ \partial t} = \frac{\beta_t \bar{\alpha}_t}{(1 - \bar{\alpha}_t)^2} \rvx - \beta_t \sqrt{\bar{\alpha}_t} \frac{ 1 + \bar{\alpha}_t }{ 2(1 - \bar{\alpha}_t)^2 } \E_{p(\rvx_0|\rvx)}[\rvx_0].
    \end{equation*}
    
    Substituting these back, we get:
    \begin{equation*}
        \begin{aligned}
            \frac{\partial g}{\partial t} & = 2 \left< - \frac{1}{1 - \bar{\alpha}_t} \rvx + \frac{\sqrt{\bar{\alpha}_t}}{1 - \bar{\alpha}_t} \E_{p(\rvx_0 | \rvx)} [\rvx_0],  \frac{\beta_t \bar{\alpha}_t}{(1 - \bar{\alpha}_t)^2} \rvx - \beta_t \sqrt{\bar{\alpha}_t} \frac{ 1 + \bar{\alpha}_t }{ 2(1 - \bar{\alpha}_t)^2 } \E_{p(\rvx_0|\rvx)}[\rvx_0]   \right>  \\ 
            & = \frac{2\beta_t\bar{\alpha}_t}{(1-\bar{\alpha}_t)^2} \left< s_t(\rvx), \rvx - \frac{1+\bar{\alpha}_t}{2\sqrt{\bar{\alpha}_t}}\E_{p(\rvx_0|\rvx)}[\rvx_0] \right>.
        \end{aligned}
    \end{equation*}    
    Next, under Cauchy-Schwarz inequality: $|\left< \rva, \rvb \right>| \leq \left\| \rva \right\| \left\| \rvb \right\|$, we have
    \begin{equation*}
        \begin{aligned}
            | \frac{\partial g}{\partial t} | & \leq \frac{2\beta_t\bar{\alpha}_t}{(1-\bar{\alpha}_t)^2} \left\| s_t(\rvx) \right\| \cdot \left\| \rvx - \frac{1+\bar{\alpha}_t}{2\sqrt{\bar{\alpha}_t}}\E_{p(\rvx_0|\rvx)}[\rvx_0] \right\| 
        \end{aligned}
    \end{equation*}
    
    Then, by the triangle inequality for vectors $\rva$ and $\rvb$:
    \begin{equation*}
        \left\| \rva - \rvb  \right\| = \left\| \rva + (- \rvb)  \right\| \leq \left\| \rva  \right\| + \left\| - \rvb \right\| = \left\| \rva \right\| + \left\| \rvb  \right\|,
    \end{equation*}
    and the assumption $\left\| \rvx \right\| \leq M$, we know that $\left\| \E_{p(\rvx_0|\rvx)}[\rvx_0] \right\| \leq M$ almost surely.
    Thus,
    \begin{equation*}
        \begin{aligned}
            | \frac{\partial g}{\partial t} | & \leq \frac{2\beta_t\bar{\alpha}_t}{(1-\bar{\alpha}_t)^2} \left\| s_t(\rvx) \right\| \cdot \left\| \rvx - \frac{1+\bar{\alpha}_t}{2\sqrt{\bar{\alpha}_t}}\E_{p(\rvx_0|\rvx)}[\rvx_0] \right\| 
            \\
            & \leq \frac{2\beta_t\bar{\alpha}_t}{(1-\bar{\alpha}_t)^2} \left\| s_t(\rvx) \right\| \cdot \left( \left\| \rvx \right\| + \left\| \frac{1+\bar{\alpha}_t}{2\sqrt{\bar{\alpha}_t}}\E_{p(\rvx_0|\rvx)}[\rvx_0] \right\| \right) \\
                                              & \leq \frac{2\beta_t\bar{\alpha}_t}{(1-\bar{\alpha}_t)^2} \left\| s_t(\rvx) \right\| \left( M + \frac{1+\bar{\alpha}_t}{2\sqrt{\bar{\alpha}_t}} M \right).
        \end{aligned}
    \end{equation*}
    
    Let $C_1 = \sup_{t \geq 0} \frac{2 \bar{\alpha}_t}{(1 - \bar{\alpha}_t)^2}$ and $C_2 = \sup_{t \geq 0} \frac{1+\bar{\alpha}_t}{2\sqrt{\bar{\alpha}_t}}$.
    As $0 < \bar{\alpha}_t < 1$, it is easy to check the maximum value of $\frac{1+\bar{\alpha}_t}{2\sqrt{\bar{\alpha}_t}}$, i.e., $C_2$ is achieved when $\bar{\alpha}_t$ gets close to $1$. 
    This also aligns with the condition where $\sup \frac{2 \bar{\alpha}_t}{(1 - \bar{\alpha}_t)^2}$ is reached.
    Thus, $\left( 1 + \frac{1+\bar{\alpha}_t}{2\sqrt{\bar{\alpha}_t}} \right)M \leq (1 + C_2) M = C_3 M$ for some constants $C_3 < 2$, since $C_2 \to 1$ when $\bar{\alpha}_t \to 1$.
    As a result, we have a constant $C > 0$ leading to the upper bound as
    \begin{equation*}
        | \frac{\partial g}{\partial t} | < C \beta_t \left\| s_t(\rvx) \right\|, \quad \text{with } C = C_1 \cdot \left(C_3 M\right).
    \end{equation*}

    Mean Value Theorem states that: for any $t_1, t_2$, there exists a $\xi$ between $t_1$ and $t_2$ such that:
    $|g(t_2) - g(t_1)| = |\frac{\partial g}{\partial t}(\xi)| |t_2 - t_1| \leq C\beta_\xi\|s_\xi(\mathbf{x})\| |t_2 - t_1|$.

    Applying this to $g(t)$, for any $t_1, t_2$, there exists a $\xi$ between $t_1$ and $t_2$, such that:
    \begin{equation*}
        \begin{aligned}
            |g(t_2) - g(t_1)| & = |\frac{\partial g}{\partial t}(\xi)| |t_2 -t_1| \\
                              & \leq C \beta_{\xi} \left\| s_{\xi}(\rvx) \right\| |t_2 - t_1| \\
                              & \leq C K \left\| s_{\xi}(\rvx) \right\| |t_2 - t_1| \qquad \text{(by assumption $\beta_t \leq K$)}
        \end{aligned}
    \end{equation*}
    Then, we can express
    \begin{equation*}
        \begin{aligned}
            | \left\| s_{t_2}(\rvx) \right\| - \left\| s_{t_1}(\rvx) \right\| | & = \frac{|\left\| s_{t_2}(\rvx) \right\|^2 - \left\| s_{t_1}(\rvx) \right\|^2 |}{(\left\| s_{t_2}(\rvx) \right\| + \left\| s_{t_1}(\rvx) \right\|)} \\
                        & \geq \frac{C K \left\| s_{\xi}(\rvx) \right\| |t_2 - t_1|}{\left\| s_{t_2}(\rvx) \right\| + \left\| s_{t_1}(\rvx) \right\|}
        \end{aligned}
    \end{equation*}
    From Lemma~\ref{lem3}, we have: there exists $C^{\prime} > 0$ and $T_0 > 0$ such that for all $t \geq T_0$: $|| s_t(\rvx)|| > C^{\prime} || \rvx ||$.
    Applying this to $t_1$, $t_2$, and $\xi$ with $\delta > 0$, we find that there exists a $C^{\prime}$ such that:
    \begin{equation*}
        \begin{aligned}
            \left\| s_{t_1}(\rvx) \right\|  & > C^{\prime} \left\| \rvx \right\| \\
            \left\| s_{t_2}(\rvx) \right\|  & > C^{\prime} \left\| \rvx \right\| \\
            \left\| s_{\xi}(\rvx) \right\|  & > C^{\prime} \left\| \rvx \right\| \\
        \end{aligned}
    \end{equation*}
    Substituting these back, we have:
    \begin{equation*}
        \begin{aligned}
        | \left\| s_{t_2}(\rvx) \right\| - \left\| s_{t_1}(\rvx) \right\| | & \geq \frac{C K_2 C^{\prime} \left\| \rvx \right\| | t_2 - t_1 | }{2 C^{\prime} \left\| \rvx \right\|} \\
                        & = \frac{CK_2}{2} |t_2 - t_1|.
        \end{aligned}
    \end{equation*}

    For any $\epsilon > 0$, let $\Delta = \frac{2\epsilon}{CK_2}$.
    Then for $|t_2 - t_1| \geq \Delta$, we have: $| \left\| s_{t_2}(\rvx) \right\| - \left\| s_{t_1}(\rvx) \right\| | > \epsilon$.
    This completes the proof.
\end{proof}

\begin{remark}[\textbf{Relationship between score norms and $t^*$}]
According to Proposition~\ref{prop}, for two distinct noise levels $t_1$ and $t_2$, we would obtain different score norms for the same sample $\rvx \in \gX$. Hence, we consider that different $t^*$ would change the score norm of the same sample $\rvx$.
In addition, we also imply that the score norms of the same sample $\rvx$ differ if and only if the noise levels $t^*$ input to the score network $s_\theta$ are different. This is because the score network only takes two input arguments $s_\theta(\rvx, t^*)$, which is sample $\rvx$ and noise level $t^*$.
For a natural example $\rvx$, we treat it as a sample with no adversarial perturbation, such that $\epsilon = 0$. From Figure~\ref{fig: eps norm with different perturbations}, cleaner samples typically exhibit smaller score norms, which means the data points are close to the high data density region \citep{song2019generative}.
Then, we should not diffuse the natural data further with a large $t^*$ to retain the original semantic information. Samples subjected to stronger perturbations exhibit larger score norms, which means the adversarial examples are far from the high data density area. We then inject certain amounts of Gaussian noise to the sample to cover its adversarial perturbation during forward diffusion.
We identify the score norm as a metric to assess the sample's proximity to the original data distribution $p(\rvx)$. 
\end{remark}

\section{Visualizations for Purification Results}
\begin{figure}[H]
    \centering
    \subcaptionbox{Adversarial Examples}{%
        \includegraphics[width=0.24\linewidth]{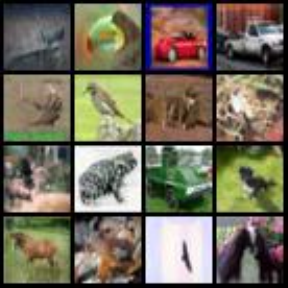}}
    \subcaptionbox{DiffPure}{%
        \includegraphics[width=0.24\linewidth]{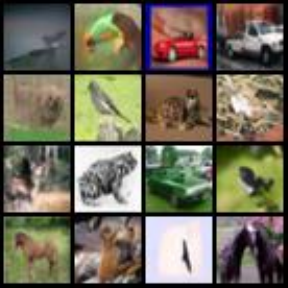}}
    \subcaptionbox{DiffPure + \emph{SSNI-L}}{%
        \includegraphics[width=0.24\linewidth]{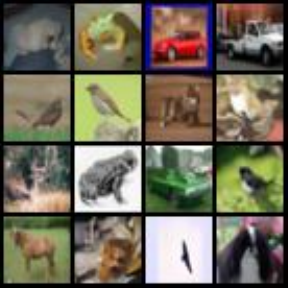}}
    \subcaptionbox{DiffPure + \emph{SSNI-N}}{%
        \includegraphics[width=0.24\linewidth]{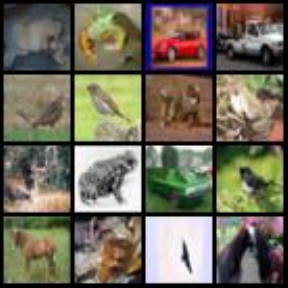}}
\end{figure}

\begin{figure}[H]
    \centering
    \subcaptionbox{Adversarial Examples}{%
        \includegraphics[width=0.24\linewidth]{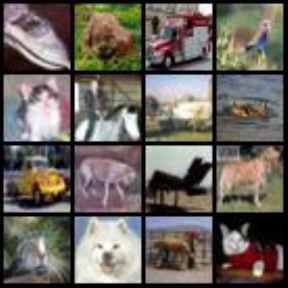}}
    \subcaptionbox{DiffPure}{%
        \includegraphics[width=0.24\linewidth]{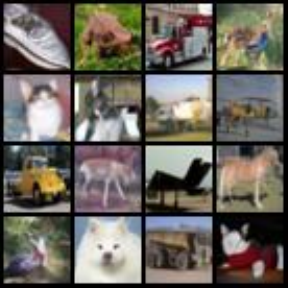}}
    \subcaptionbox{DiffPure + \emph{SSNI-L}}{%
        \includegraphics[width=0.24\linewidth]{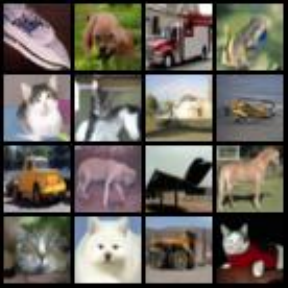}}
    \subcaptionbox{DiffPure + \emph{SSNI-N}}{%
        \includegraphics[width=0.24\linewidth]{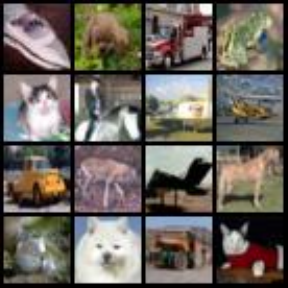}}
\end{figure}

\begin{figure}[H]
    \centering
    \subcaptionbox{Natural Examples}{%
        \includegraphics[width=0.24\linewidth]{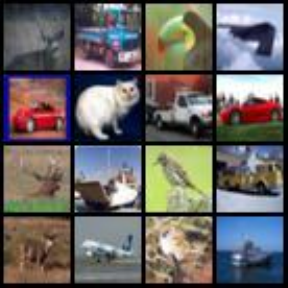}}
    \subcaptionbox{DiffPure}{%
        \includegraphics[width=0.24\linewidth]{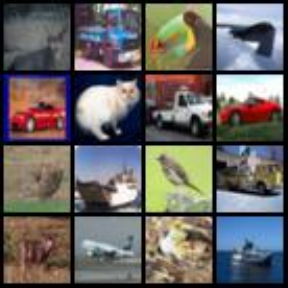}}
    \subcaptionbox{DiffPure + \emph{SSNI-L}}{%
        \includegraphics[width=0.24\linewidth]{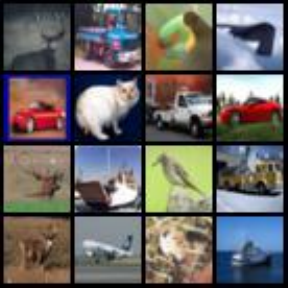}}
    \subcaptionbox{DiffPure + \emph{SSNI-N}}{%
        \includegraphics[width=0.24\linewidth]{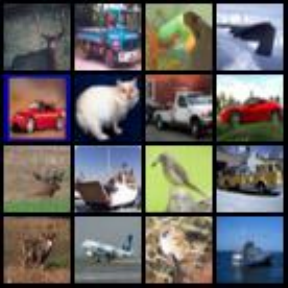}}
\end{figure}


\end{document}

%% file: math_commands.tex
\usepackage{amsmath,amsfonts,bm}

\def\eqref#1{Eq.~(\ref{#1})}

\def\1{\bm{1}}

\def\rva{{\mathbf{a}}}
\def\rvb{{\mathbf{b}}}

\def\rvt{{\mathbf{t}}}

\def\rvx{{\mathbf{x}}}
\def\rvy{{\mathbf{y}}}

\def\rmI{{\mathbf{I}}}

\def\vx{{\bm{x}}}

\DeclareMathAlphabet{\mathsfit}{\encodingdefault}{\sfdefault}{m}{sl}
\SetMathAlphabet{\mathsfit}{bold}{\encodingdefault}{\sfdefault}{bx}{n}

\def\gL{{\mathcal{L}}}

\def\gN{{\mathcal{N}}}

\def\gT{{\mathcal{T}}}

\def\gX{{\mathcal{X}}}
\def\gY{{\mathcal{Y}}}

\def\sR{{\mathbb{R}}}

\newcommand{\E}{\mathbb{E}}



%% file: algorithm/DBP-SSNI.tex
\begin{algorithm}[t]
	\renewcommand{\algorithmicrequire}{\textbf{Input:}}
	\renewcommand{\algorithmicensure}{\textbf{Output:}}
	\caption{Diffusion-based Purification with SSNI.}
	\label{alg1:SSNI}
	\begin{algorithmic}[1]
		\REQUIRE test samples $\mathbf{x}$, a score network $s_{\theta}$, a reweighting function $f(\cdot)$, a noise level $t^{\rm S}$ for score evaluation, and a pre-determined noise level $t^*$.
            \STATE Approximate the score by $s_\theta$: $s_\theta(\rvx, t^{\rm S})$
            \STATE Obtain the sample-specific noise level: $t(\rvx) = f(\left\|s_\theta(\rvx, t^{\rm S})\right\|, t^*)$
            \STATE Execute forward diffusion process $\rvx_{t(\rvx)} \gets$ Eq.~(\ref{eq: sample-specific noise infusion})
            \FOR{$\rvt = t(\rvx), \dots, \textbf{1}$}
            \STATE Execute Reserve diffusion process $\hat{\rvx}_{t-1} \gets $ Eq.~(\ref{eq: reverse})
            \ENDFOR
		\STATE \textbf{return} purified samples $\hat{\mathbf{x}}$
	\end{algorithmic}
\end{algorithm}

%% file: tables/CIFAR10-MAIN.tex
\begin{table*}[t]
    \centering
    \caption{Standard and robust accuracy of DBP methods against adaptive white-box PGD+EOT (left: $\ell_\infty (\epsilon = 8/255)$, right: $\ell_2 (\epsilon = 0.5)$) on \emph{CIFAR-10}. WideResNet-28-10 and WideResNet-70-16 are used as classifiers. We compare the result of DBP methods with and without \emph{SSNI-N}. We report mean and standard deviation over three runs. We show the most successful defense in \textbf{bold}. The performance improvements and degradation are reported in \textcolor{dg}{green} and \textcolor{dr}{red}.}
    \label{table:best_cifar10}
    \vspace{2pt}
    \begin{subtable}{.49\linewidth}
        \centering
        \resizebox{\textwidth}{!}{%
        \begin{tabular}{clcc}
        \toprule
        \multicolumn{4}{c}{PGD+EOT $\ell_\infty~(\epsilon = 8/255)$}\\
        \midrule
        & DBP Method & Standard & Robust \\ \midrule 
        \multirow{7}{*}{\rotatebox[origin=c]{90}{WRN-28-10}}
        & \citet{nie2022diffusion} & 89.71$\pm$0.72 & 47.98$\pm$0.64 \\
        & + \cellcolor{lg}{\emph{SSNI-N}} & \cellcolor{lg}{\bf{93.29$\pm$0.37 \textcolor{dg}{(+3.58)}}} & \cellcolor{lg}{\bf{48.63$\pm$0.56} \textcolor{dg}{(+0.65)}} \\
        \cmidrule{2-4}
        & \citet{wang2022guided} & 92.45$\pm$0.64 & 36.72$\pm$1.05 \\
        & + \cellcolor{lg}{\emph{SSNI-N}} & \cellcolor{lg}{\bf{94.08$\pm$0.33 \textcolor{dg}{(+1.63)}}} & \cellcolor{lg}{\bf{40.95$\pm$0.65 \textcolor{dg}{(+4.23)}}} \\
        \cmidrule{2-4}
        & \citet{lee2023robust} & 90.10$\pm$0.18 & 56.05$\pm$1.11 \\
        & \cellcolor{lg}{+ \emph{SSNI-N}} & \cellcolor{lg}{\bf{93.55$\pm$0.55 \textcolor{dg}{(+3.45)}}} & \cellcolor{lg}{\bf{56.45$\pm$0.28 \textcolor{dg}{(+0.40)}}} \\
        \midrule

           \multirow{7}{*}{\rotatebox[origin=c]{90}{WRN-70-16}} & \citet{nie2022diffusion} & 90.89$\pm$1.13 & 52.15$\pm$0.30 \\
        & \cellcolor{lg}{+ \emph{SSNI-N}} & \cellcolor{lg}{\bf{94.47$\pm$0.51 \textcolor{dg}{(+3.58)}}} & \cellcolor{lg}{\bf{52.47$\pm$0.66 \textcolor{dg}{(+0.32)}}} \\
        \cmidrule{2-4}
        & \citet{wang2022guided} & 93.10$\pm$0.51 & 43.55$\pm$0.58 \\
        & \cellcolor{lg}{+ \emph{SSNI-N}} & \cellcolor{lg}{\bf{95.57$\pm$0.24 \textcolor{dg}{(+2.47)}}} & \cellcolor{lg}{\bf{46.03$\pm$1.33 \textcolor{dg}{(+2.48)}}} \\
        \cmidrule{2-4}
        & \citet{lee2023robust} & 89.39$\pm$1.12 & 56.97$\pm$0.33 \\
        & \cellcolor{lg}{+ \emph{SSNI-N}} & \cellcolor{lg}{\bf{93.82$\pm$0.24 \textcolor{dg}{(+4.43)}}} & \cellcolor{lg}{\bf{57.03$\pm$0.28 \textcolor{dg}{(+0.06)}}} \\
        \bottomrule
        \end{tabular}
        }
    \end{subtable}
    \begin{subtable}{.49\linewidth}
        \centering
        \resizebox{\textwidth}{!}{%
        \begin{tabular}{clcc}
        \toprule
        \multicolumn{4}{c}{PGD+EOT $\ell_2~(\epsilon = 0.5)$}\\
        \midrule
        & DBP Method & Standard & Robust \\ \midrule 
        \multirow{7}{*}{\rotatebox[origin=c]{90}{WRN-28-10}} & \citet{nie2022diffusion} & 91.80$\pm$0.84 & \bf{82.81$\pm$0.97} \\
        & \cellcolor{lg}{+ \emph{SSNI-N}} & \cellcolor{lg}{\bf{93.95$\pm$0.70 \textcolor{dg}{(+2.15)}}} & \cellcolor{lg}{82.75$\pm$1.01 \textcolor{dr}{(-0.06)}} \\
        \cmidrule{2-4}
        & \citet{wang2022guided} & 92.45$\pm$0.64 & 82.29$\pm$0.82 \\
        & \cellcolor{lg}{+ \emph{SSNI-N}} & \cellcolor{lg}{\bf{94.08$\pm$0.33 \textcolor{dg}{(+1.63)}}} & \cellcolor{lg}{\bf{82.49$\pm$0.75 \textcolor{dg}{(+0.20)}}} \\
        \cmidrule{2-4}
        & \citet{lee2023robust} & 90.10$\pm$0.18 & 83.66$\pm$0.46 \\
        & \cellcolor{lg}{+ \emph{SSNI-N}} & \cellcolor{lg}{\bf{93.55$\pm$0.55 \textcolor{dg}{(+3.45)}}} & \cellcolor{lg}{\bf{84.05$\pm$0.33 \textcolor{dg}{(+0.39)}}} \\
        \midrule
        \multirow{7}{*}{\rotatebox[origin=c]{90}{WRN-70-16}} & \citet{nie2022diffusion} & 92.90$\pm$0.40 & 82.94$\pm$1.13 \\
        & \cellcolor{lg}{+ \emph{SSNI-N}} & \cellcolor{lg}{\bf{95.12$\pm$0.58 \textcolor{dg}{(+2.22)}}} & \cellcolor{lg}{\bf{84.38$\pm$0.58 \textcolor{dg}{(+1.44)}}} \\
        \cmidrule{2-4}
        & \citet{wang2022guided} & 93.10$\pm$0.51 & \bf{85.03$\pm$0.49} \\
        & \cellcolor{lg}{+ \emph{SSNI-N}} & \cellcolor{lg}{\bf{95.57$\pm$0.24 \textcolor{dg}{(+2.47)}}} & \cellcolor{lg}{84.64$\pm$0.51 \textcolor{dr}{(-0.39)}} \\
        \cmidrule{2-4}
        & \citet{lee2023robust} & 89.39$\pm$1.12 & 84.51$\pm$0.37 \\
        & \cellcolor{lg}{+ \emph{SSNI-N}} & \cellcolor{lg}{\bf{93.82$\pm$0.24 \textcolor{dg}{(+4.43)}}} & \cellcolor{lg}{\bf{84.83$\pm$0.33 \textcolor{dg}{(+0.32)}}} \\
        \bottomrule
        \end{tabular}
        }
    \end{subtable}
\end{table*}

%% file: tables/ImageNet-SUB.tex
\begin{table}[t]
\centering
\caption{Standard and robust accuracy (\%) against adaptive white-box PGD+EOT $\ell_\infty (\epsilon = 4/255)$ on \emph{ImageNet-1K}.}
\label{table:linf_imagenet_resnet50}
\vspace{2pt}
\centering
\resizebox{\linewidth}{!}{%
\begin{tabular}{clcc}
\toprule
\multicolumn{4}{c}{PGD+EOT $\ell_\infty~(\epsilon = 4/255)$}\\
\midrule
& DBP Method & Standard & Robust \\ 
\midrule
\multirow{7}{*}{\rotatebox[origin=c]{90}{RN-50}} &
\citet{nie2022diffusion} & 68.23$\pm$0.92 & 30.34$\pm$0.72 \\
& \cellcolor{lg}{+ \emph{SSNI-N}} & \cellcolor{lg}{\bf{70.25$\pm$0.56 \textcolor{dg}{(+2.02)}}} & \cellcolor{lg}{\bf{33.66$\pm$1.04 \textcolor{dg}{(+3.32)}}} \\
\cmidrule{2-4}
& \citet{wang2022guided} & 74.22$\pm$0.12 & 0.39$\pm$0.03 \\
& \cellcolor{lg}{+ \emph{SSNI-N}} & \cellcolor{lg}{\bf{75.07$\pm$0.18 \textcolor{dg}{(+0.85)}}} & \cellcolor{lg}{\bf{5.21$\pm$0.24 \textcolor{dg}{(+4.82)}}} \\
\cmidrule{2-4}
& \citet{lee2023robust} & 70.18$\pm$0.60 & 42.45$\pm$0.92 \\
& \cellcolor{lg}{+ \emph{SSNI-N}} & \cellcolor{lg}{\bf{72.69$\pm$0.80 \textcolor{dg}{(+2.51)}}} & \cellcolor{lg}{\bf{43.48$\pm$0.25 \textcolor{dg}{(+1.03)}}} \\
\bottomrule
\end{tabular}
}
\end{table}

%% file: tables/BPDA.tex
\begin{table}[t]
\centering
\caption{Standard and robust accuracy (\%) against adaptive white-box BPDA+EOT $\ell_\infty (\epsilon = 8/255)$ attack on \emph{CIFAR-10}.}
\label{table:bpda2810}
\vspace{2pt}
\resizebox{\linewidth}{!}{%
\begin{tabular}{clcc}
\toprule
\multicolumn{4}{c}{BPDA+EOT $\ell_\infty~(\epsilon = 8/255)$}\\
\midrule
& DBP Method & Standard & Robust \\
\midrule
\multirow{7}{*}{\rotatebox[origin=c]{90}{WRN-28-10}} & \citet{nie2022diffusion} & 89.71$\pm$0.72&81.90$\pm$0.49\\
& \cellcolor{lg}{+ \emph{SSNI-N}} & \cellcolor{lg}{\bf{93.29$\pm$0.37 \textcolor{dg}{(+3.58)}}} & \cellcolor{lg}{\bf{82.10$\pm$1.15 \textcolor{dg}{(+0.20)}}} \\
\cmidrule{2-4}
& \citet{wang2022guided} & 92.45$\pm$0.64 & 79.88$\pm$0.89 \\
& \cellcolor{lg}{+ \emph{SSNI-N}} & \cellcolor{lg}{\bf{94.08$\pm$0.33 \textcolor{dg}{(+1.63)}}} & \cellcolor{lg}{\bf{80.99$\pm$1.09 \textcolor{dg}{(+1.11)}}} \\
\cmidrule{2-4}
& \citet{lee2023robust} & 90.10$\pm$0.18 & \bf{88.40$\pm$0.88} \\
& \cellcolor{lg}{+ \emph{SSNI-N}} & \cellcolor{lg}{\bf{93.55$\pm$0.55 \textcolor{dg}{(+3.45)}}} & \cellcolor{lg}{87.30$\pm$0.42 \textcolor{dr}{(-1.10)}} \\
\bottomrule
\end{tabular}
}
\end{table}

%% file: tables/Diff_sampling_methods.tex
\begin{table}[t]
\centering
\caption{Ablation study on different sampling methods during the reverse diffusion process. We measure the standard and robust accuracy (\%) against PGD+EOT $\ell_\infty (\epsilon = 8/255)$ on \emph{CIFAR-10}. We use \emph{DiffPure} as the baseline method and we set $t^*=100$. WideResNet-28-10 is used as the classifier. We report mean and the standard deviations over three runs.}
\vspace{2pt}
\label{table:sampling methods}
\footnotesize
\begin{tabular}{ccc}
\toprule
Sampling Method & Standard & Robust \\
\midrule
sdeint solver &89.06$\pm$0.48&47.72$\pm$0.24\\
DDPM&89.71$\pm$0.72&47.98$\pm$0.64\\
DDIM &91.54$\pm$0.72&37.50$\pm$0.80\\
\bottomrule
\end{tabular}
\end{table}

%% file: tables/tau.tex
\begin{figure}[t]
\centering
\includegraphics[width=0.9\linewidth]{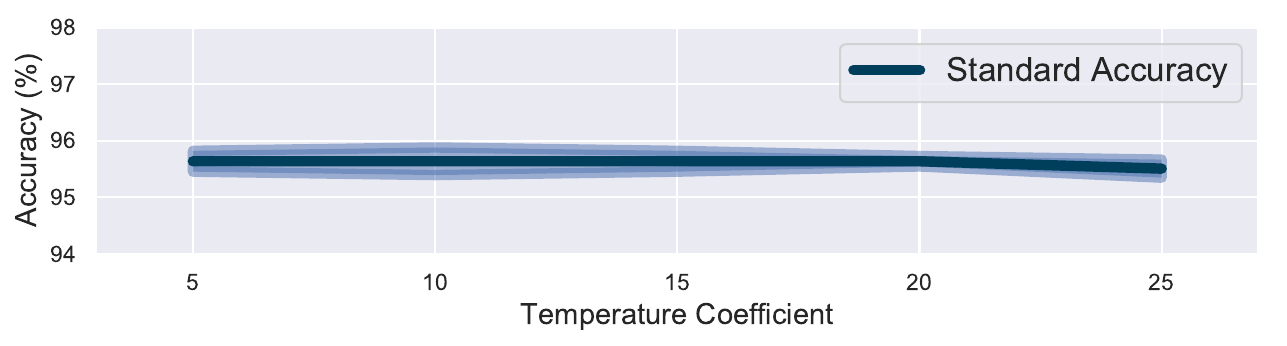}
\includegraphics[width=0.9\linewidth]{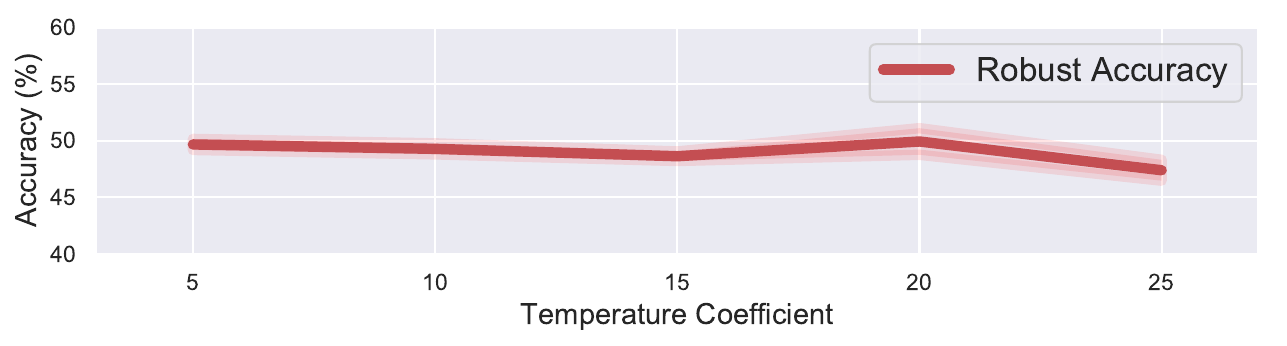}
\caption{Standard (\textbf{top}) and robust ({\bf bottom}) accuracy (\%) vs. $\tau$; We report mean and the standard deviations over three runs.}
\label{fig: ablation linf}
\end{figure}

%% file: tables/Inference_Time_CIFAR10.tex
\begin{table*}[t]
\centering
\caption{Inference time of the DBP methods with and without SSNI for a single image running on one A100 GPU on \emph{CIFAR-10} and \emph{ImageNet-1K}. We use WideResNet-28-10 as the classifier for \emph{CIFAR-10} and ResNet-50 for \emph{ImageNet-1K}.}
\label{table: inference time cifar10}
\vspace{2pt}
\begin{subtable}{.45\linewidth}
\centering
\resizebox{\textwidth}{!}{%
\begin{tabular}{lcc}
\toprule
DBP Method & Noise Injection Method & Time (s)\\
\midrule
                               &-      & 3.934       \\
\citet{nie2022diffusion}       &SSNI-L & 4.473       \\
                               &SSNI-N & 4.474       \\
\midrule
                               & -     & 5.174       \\
\citet{wang2022guided}         &SSNI-L & 5.793       \\
                               &SSNI-N & 5.829       \\
\midrule
                               & -     & 14.902      \\
\citet{lee2023robust}          &SSNI-L & 15.624      \\
                               &SSNI-N & 15.534      \\
\bottomrule
\end{tabular}
}
\end{subtable}
\begin{subtable}{.45\linewidth}
\centering
\resizebox{\textwidth}{!}{%
\begin{tabular}{lcc}
\toprule
DBP Method & Noise Injection Method & Time (s) \\
\midrule
                               &-      & 8.980       \\
\citet{nie2022diffusion}       &SSNI-L & 14.515       \\
                               &SSNI-N & 14.437       \\
\midrule
                               & -     & 11.271       \\
\citet{wang2022guided}         &SSNI-L & 16.657       \\
                               &SSNI-N & 16.747       \\
\midrule
                               & -     & 35.091      \\
\citet{lee2023robust}          &SSNI-L & 40.526      \\
                               &SSNI-N & 40.633      \\
\bottomrule
\end{tabular}
}
\end{subtable}
\end{table*}

%% file: algorithm/PGD+EOT-adaptive.tex
\begin{algorithm}[H]
\caption{Adaptive white-box PGD+EOT attack for SSNI.}
\label{alg:pgd_eot}
\begin{algorithmic}[1]
\REQUIRE clean data-label pairs ($\rvx, y)$; purifier $f_p$; classifier $f_c$; a noise level $T$; a score network $s_{\theta}(\rvx,T)$; objective function $\gL$; perturbation budget $\epsilon$; step size $\alpha$; PGD iterations $K$; EOT iterations $N$.

\STATE Initialize $\rvx_0^{\text{adv}} \leftarrow \rvx$
\FOR{$k = 0,...,K - 1$}
    \STATE Computing sample-specific noise levels: $t(\rvx^{\text{adv}}_k) \leftarrow f(\left\|s_\theta(\rvx^{\text{adv}}_k, T)\right\|, T)$
    \STATE Average the gradients over EOT: $g_k \leftarrow \frac{1}{N} \sum_{i=1}^N \nabla_{\rvx^{\rm adv}} \gL\left(f_c(f_p(\rvx_k^{\text{adv}}, t(\rvx^{\text{adv}}_k))), y\right)$
    \STATE Update adversarial examples: $\rvx_{k+1}^{\text{adv}} \leftarrow \Pi_{\mathcal{B}_{\epsilon}(\rvx)} \left( \rvx_k^{\text{adv}} + \alpha \cdot \text{sign}(g_k) \right)$
\ENDFOR
\STATE return $\rvx^{\text{adv}} = \rvx_K^{\text{adv}}$
\end{algorithmic}
\end{algorithm}

%% file: algorithm/BPDA+EOT-adaptive.tex
\begin{algorithm}[H]
\caption{Adaptive white-box BPDA+EOT attack.}
\label{alg:bpda_eot}
\begin{algorithmic}[1]
\REQUIRE clean data-label pairs ($\rvx, y)$; purifier $f_p$; classifier $f_c$; approximation function $f_{\text{apx}}$; a noise level $T$; a score network $s_{\theta}(\rvx,T)$; objective function $\gL$; perturbation budget $\epsilon$; step size $\alpha$; PGD iterations $K$; EOT iterations $N$.
\STATE Initialize $\rvx_0^{\text{adv}} \leftarrow \rvx$\;
\FOR{$k \leftarrow 0$ to $K-1$}
    \STATE Computing sample-specific $t$:  $t(\rvx^{\text{adv}}_k) \leftarrow f(\left\|s_\theta(\rvx^{\text{adv}}_k, T)\right\|, T)$
    \STATE Average the gradient over EOT samples: $ g_k \leftarrow \frac{1}{N} \sum_{i=1}^N \nabla_{\rvx^{\rm adv}} \gL\left((f_c(f_{\text{apx}}(f_p(\rvx_k^{\text{adv}})))), y\right)$
    \STATE Update adversarial examples: $\rvx_{k+1}^{\text{adv}} \leftarrow \Pi_{\mathcal{B}_{\epsilon}(\rvx)} \left( \rvx_k^{\text{adv}} + \alpha \cdot \text{sign}(g_k) \right)$
\ENDFOR
\STATE \textbf{return} $\rvx^{\text{adv}} = \rvx_K^{\text{adv}}$\;
\end{algorithmic}
\end{algorithm}

%% file: tables/CIFAR10-SUB-SSNI-L.tex
\begin{table*}[htbp]
    \centering
    \caption{Standard and robust accuracy of DBP methods against adaptive white-box PGD+EOT (left: $\ell_\infty (\epsilon = 8/255)$, right: $\ell_2 (\epsilon = 0.5)$) on \emph{CIFAR-10}. WideResNet-28-10 and WideResNet-70-16 are used as classifiers. We compare the result of DBP methods with and without \emph{SSNI-L}. We report mean and standard deviation over three runs. We show the most successful defense in \textbf{bold}.}
    \vspace{2pt}
    \label{table:cifar10 SSNI-L SUB}
    \footnotesize
    \begin{subtable}{0.49\linewidth}
        \centering
        \begin{tabular}{clcc}
        \toprule
        \multicolumn{4}{c}{PGD+EOT $\ell_\infty~(\epsilon = 8/255)$}\\
        \midrule
        & DBP Method & Standard & Robust \\ \midrule
        \multirow{7}{*}{\rotatebox[origin=c]{90}{WRN-28-10}}
        & \citet{nie2022diffusion} & 89.71$\pm$0.72 & \bf{47.98$\pm$0.64} \\
        & + \cellcolor{lg}{\emph{SSNI-L}} & \cellcolor{lg}{\bf{92.97$\pm$0.42}} & \cellcolor{lg}{46.35$\pm$0.72} \\
        \cmidrule{2-4}
        & \citet{wang2022guided} & 92.45$\pm$0.64 & \bf{36.72$\pm$1.05} \\
        & + \cellcolor{lg}{\emph{SSNI-L}} & \cellcolor{lg}{\bf{93.62$\pm$0.49}} & \cellcolor{lg}{36.59$\pm$1.29} \\
        \cmidrule{2-4}
        & \citet{lee2023robust} & 90.1$\pm$0.18 & \bf{56.05$\pm$1.11} \\
        & + \cellcolor{lg}{\emph{SSNI-L}} &  \cellcolor{lg}{\bf{93.49$\pm$0.33}} &  \cellcolor{lg}{53.71$\pm$0.48} \\
        \midrule

        \multirow{7}{*}{\rotatebox[origin=c]{90}{WRN-70-16}}
        & \citet{nie2022diffusion} & 90.89$\pm$1.13 & \bf{52.15$\pm$2.30} \\
        & + \cellcolor{lg}{\emph{SSNI-L}} & \cellcolor{lg}{\bf{93.82$\pm$0.49}} & \cellcolor{lg}{49.94$\pm$0.33} \\
        \cmidrule{2-4}
        & \citet{wang2022guided} & 93.10$\pm$0.51 & \bf{43.55$\pm$0.58} \\
        & + \cellcolor{lg}{\emph{SSNI-L}} & \cellcolor{lg}{\bf{93.88$\pm$0.49}} & \cellcolor{lg}{43.03$\pm$0.60} \\
        \cmidrule{2-4}
        & \citet{lee2023robust} & 89.39$\pm$1.12 & \bf{56.97$\pm$0.33} \\
        & + \cellcolor{lg}{\emph{SSNI-L}} & \cellcolor{lg}{\bf{92.64$\pm$0.40}} & \cellcolor{lg}{52.86$\pm$0.46} \\
        \bottomrule
        \end{tabular}
    \end{subtable}
    \begin{subtable}{.49\linewidth}
        \footnotesize
        \centering
        \begin{tabular}{clcc}
        \toprule
        \multicolumn{4}{c}{PGD+EOT $\ell_2~(\epsilon = 0.5)$}\\
        \midrule
        & DBP Method & Standard & Robust \\ \midrule 
        \multirow{7}{*}{\rotatebox[origin=c]{90}{WRN-28-10}}
        & \citet{nie2022diffusion} & 91.80$\pm$0.84 & \bf{82.81$\pm$0.97} \\
        & + \cellcolor{lg}{\emph{SSNI-L}} & \cellcolor{lg}{\bf{93.82$\pm$0.37}} & \cellcolor{lg}{81.12$\pm$0.80} \\
        \cmidrule{2-4}
        & \citet{wang2022guided} & 92.45$\pm$0.64 & \bf{82.29$\pm$0.82} \\
        & + \cellcolor{lg}{\emph{SSNI-L}} & \cellcolor{lg}{\bf{93.62$\pm$0.49}} & \cellcolor{lg}{80.66$\pm$1.31} \\
        \cmidrule{2-4}
        & \citet{lee2023robust} & 90.10$\pm$0.18 & 83.66$\pm$0.46 \\
        & + \cellcolor{lg}{\emph{SSNI-L}} & \cellcolor{lg}{\bf{93.49$\pm$0.33}} & \cellcolor{lg}{\bf{85.29$\pm$0.24}} \\
        \midrule

        \multirow{7}{*}{\rotatebox[origin=c]{90}{WRN-70-16}}
        & \citet{nie2022diffusion} & 92.90$\pm$0.40 & 82.94$\pm$1.13 \\
        & + \cellcolor{lg}{\emph{SSNI-L}} & \cellcolor{lg}{\bf{94.99$\pm$0.24}} & \cellcolor{lg}{\bf{84.44$\pm$0.56}} \\
        \cmidrule{2-4}
        & \citet{wang2022guided} & 93.10$\pm$0.51 & \bf{85.03$\pm$0.49} \\
        & + \cellcolor{lg}{\emph{SSNI-L}} & \cellcolor{lg}{\bf{93.88$\pm$0.49}} & \cellcolor{lg}{82.88$\pm$0.79} \\
        \cmidrule{2-4}
        & \citet{lee2023robust} & 89.39$\pm$1.12 & 84.51$\pm$0.37 \\
        & + \cellcolor{lg}{\emph{SSNI-L}} & \cellcolor{lg}{\bf{92.64$\pm$0.40}} & \cellcolor{lg}{\bf{84.90$\pm$0.09}} \\
        \bottomrule
        \end{tabular}
    \end{subtable}
\end{table*}

%% file: tables/BPDA-2810_SSNI-L.tex
\begin{table}[htbp]
\centering
\caption{Standard and robust accuracy (\%) against adaptive white-box BPDA+EOT $\ell_\infty (\epsilon = 8/255)$ attack on \emph{CIFAR-10}. We compare the result of DBP methods with and without \emph{SSNI-L}. We report mean and standard deviation over three runs. We show the most successful defense in \textbf{bold}.}
\footnotesize
\begin{tabular}{clcc}
\toprule
\multicolumn{4}{c}{BPDA+EOT $\ell_\infty~(\epsilon = 8/255)$}\\
\midrule
& DBP Method & Standard & Robust \\
\midrule
\multirow{7}{*}{\rotatebox[origin=c]{90}{WRN-28-10}} & \citet{nie2022diffusion} & 89.71$\pm$0.72& \bf{81.90$\pm$0.49}\\
& \cellcolor{lg}{+ \emph{SSNI-L}} & \cellcolor{lg}{\bf{92.97$\pm$0.42}} & \cellcolor{lg}{80.08$\pm$0.96} \\
\cmidrule{2-4}
& \citet{wang2022guided} & 92.45$\pm$0.64 & 79.88$\pm$0.89 \\
& \cellcolor{lg}{+ \emph{SSNI-L}} & \cellcolor{lg}{\bf{93.62$\pm$0.49}} & \cellcolor{lg}{\bf{79.95$\pm$1.12}} \\
\cmidrule{2-4}
& \citet{lee2023robust} & 90.10$\pm$0.18 & 88.40$\pm$0.88 \\
& \cellcolor{lg}{+ \emph{SSNI-L}} & \cellcolor{lg}{\bf{93.49$\pm$0.33}} & \cellcolor{lg}{\bf{88.41$\pm$0.09}} \\
\bottomrule
\end{tabular}
\end{table}

%% file: tables/CIFAR10-SUB-Score.tex
\vspace{-1em}
\begin{table}[!htbp]
    \centering
    \caption{Standard and robust accuracy (\%) against adaptive white-box PGD+EOT $\ell_\infty (\epsilon = 8/255)$ attack on \emph{CIFAR-10}. We use \emph{single score norms} (i.e., $\left\|\nabla_{\rvx}{\log p_t(\rvx)}\right\|$). We report mean and standard deviation over three runs. We show the most successful defense in \textbf{bold}.}
    \vspace{2pt}
    \label{table:score norm}
    \centering
    \footnotesize
    \begin{tabular}{clcc}
    \toprule
    \multicolumn{4}{c}{PGD+EOT $\ell_\infty~(\epsilon = 8/255)$}\\
    \midrule
    & DBP Method & Standard & Robust \\ \midrule
    \multirow{7}{*}{\rotatebox[origin=c]{90}{WRN-28-10}} & \citet{nie2022diffusion}       &89.71$\pm$0.72&\bf{47.98$\pm$0.64}\\
    & + \cellcolor{lg}{\emph{SSNI-L}} &\cellcolor{lg}{91.31$\pm$0.24}&\cellcolor{lg}{46.92$\pm$0.52}\\
    & + \cellcolor{lg}{\emph{SSNI-N}} &\cellcolor{lg}{\bf{92.84$\pm$0.18}}&\cellcolor{lg}{47.20$\pm$1.22}\\
    \cmidrule{2-4}
    & \citet{wang2022guided}          & 92.45$\pm$0.64 & \bf{36.72$\pm$1.05} \\
     & + \cellcolor{lg}{\emph{SSNI-L}} &\cellcolor{lg}{93.15$\pm$0.92}&\cellcolor{lg}{35.72$\pm$1.33}\\
    & + \cellcolor{lg}{\emph{SSNI-N}}   & \cellcolor{lg}{\bf{93.42$\pm$0.60}}& \cellcolor{lg}{34.24$\pm$1.45 }\\
    \cmidrule{2-4}
    & \citet{lee2023robust}           & 90.10$\pm$0.18 & \bf{56.05$\pm$1.11} \\
    & + \cellcolor{lg}{\emph{SSNI-L}} &\cellcolor{lg}{93.40$\pm$0.49}&\cellcolor{lg}{54.52$\pm$0.46}\\
    & + \cellcolor{lg}{\emph{SSNI-N}}    & \cellcolor{lg}{\bf{93.55$\pm$0.42}} & \cellcolor{lg}{55.47$\pm$1.15} \\
    \bottomrule
    \end{tabular}
\end{table}

%% file: example_paper.bib
@inproceedings{szegedy2014intriguing,
  author       = {Christian Szegedy and
                  Wojciech Zaremba and
                  Ilya Sutskever and
                  Joan Bruna and
                  Dumitru Erhan and
                  Ian J. Goodfellow and
                  Rob Fergus},
  title        = {Intriguing properties of neural networks},
  booktitle    = {ICLR},
  year         = {2014}
}

@inproceedings{goodfellow2015explaining,
  author       = {Ian J. Goodfellow and
                  Jonathon Shlens and
                  Christian Szegedy},
  title        = {Explaining and Harnessing Adversarial Examples},
  booktitle    = {ICLR},
  year         = {2015}
}

@inproceedings{cao2021invisible,
  author       = {Yulong Cao and
                  Ningfei Wang and
                  Chaowei Xiao and
                  Dawei Yang and
                  Jin Fang and
                  Ruigang Yang and
                  Qi Alfred Chen and
                  Mingyan Liu and
                  Bo Li},
  title        = {Invisible for both Camera and LiDAR: Security of Multi-Sensor Fusion based Perception in Autonomous Driving Under Physical-World Attacks},
  booktitle    = {{IEEE} Symposium on Security and Privacy},
  pages        = {176--194},
  year         = {2021},
}

@inproceedings{dong2019efficient,
  author       = {Yinpeng Dong and
                  Hang Su and
                  Baoyuan Wu and
                  Zhifeng Li and
                  Wei Liu and
                  Tong Zhang and
                  Jun Zhu},
  title        = {Efficient Decision-Based Black-Box Adversarial Attacks on Face Recognition},
  booktitle    = {CVPR},
  year         = {2019},
}

@inproceedings{jing2021too,
  author       = {Pengfei Jing and
                  Qiyi Tang and
                  Yuefeng Du and
                  Lei Xue and
                  Xiapu Luo and
                  Ting Wang and
                  Sen Nie and
                  Shi Wu},
  title        = {Too Good to Be Safe: Tricking Lane Detection in Autonomous Driving with Crafted Perturbations},
  booktitle    = {{USENIX} Security Symposium},
  pages        = {3237--3254},
  year         = {2021},
}

@inproceedings{madry2018towards,
  author       = {Aleksander Madry and
                  Aleksandar Makelov and
                  Ludwig Schmidt and
                  Dimitris Tsipras and
                  Adrian Vladu},
  title        = {Towards Deep Learning Models Resistant to Adversarial Attacks},
  booktitle    = {ICLR},
  year         = {2018}
}

@inproceedings{zhang2019theoretically,
  author       = {Hongyang Zhang and
                  Yaodong Yu and
                  Jiantao Jiao and
                  Eric P. Xing and
                  Laurent El Ghaoui and
                  Michael I. Jordan},
  title        = {Theoretically Principled Trade-off between Robustness and Accuracy},
  booktitle    = {ICML},
  year         = {2019}
}

@inproceedings{wang2020improving,
title={Improving Adversarial Robustness Requires Revisiting Misclassified Examples},
author={Yisen Wang and Difan Zou and Jinfeng Yi and James Bailey and Xingjun Ma and Quanquan Gu},
booktitle={ICLR},
year={2020}
}

@inproceedings{yoon2021adversarial,
  author       = {Jongmin Yoon and
                  Sung Ju Hwang and
                  Juho Lee},
  title        = {Adversarial Purification with Score-based Generative Models},
  booktitle    = {ICML},
  year         = {2021},
}

@inproceedings{nie2022diffusion,
  author       = {Weili Nie and
                  Brandon Guo and
                  Yujia Huang and
                  Chaowei Xiao and
                  Arash Vahdat and
                  Animashree Anandkumar},
  title        = {Diffusion Models for Adversarial Purification},
  booktitle    = {ICML},
  year         = {2022},
}

@inproceedings{xiao2023densepure,
  author       = {Chaowei Xiao and
                  Zhongzhu Chen and
                  Kun Jin and
                  Jiongxiao Wang and
                  Weili Nie and
                  Mingyan Liu and
                  Anima Anandkumar and
                  Bo Li and
                  Dawn Song},
  title        = {DensePure: Understanding Diffusion Models for Adversarial Robustness},
  booktitle    = {ICLR},
  year         = {2023},
}

@inproceedings{lee2023robust,
  author       = {Minjong Lee and
                  Dongwoo Kim},
  title        = {Robust Evaluation of Diffusion-Based Adversarial Purification},
  booktitle    = {ICCV},
  year         = {2023},
}

@inproceedings{ho2020denoising,
  author       = {Jonathan Ho and
                  Ajay Jain and
                  Pieter Abbeel},
  title        = {Denoising Diffusion Probabilistic Models},
  booktitle    = {NeurIPS},
  year         = {2020},
}

@inproceedings{song2021score,
  author       = {Yang Song and
                  Jascha Sohl{-}Dickstein and
                  Diederik P. Kingma and
                  Abhishek Kumar and
                  Stefano Ermon and
                  Ben Poole},
  title        = {Score-Based Generative Modeling through Stochastic Differential Equations},
  booktitle    = {ICLR},
  year         = {2021},
}

@inproceedings{song2019generative,
  author       = {Yang Song and
                  Stefano Ermon},
  title        = {Generative Modeling by Estimating Gradients of the Data Distribution},
  booktitle    = {NeurIPS},
  year         = {2019},
}

@article{cifar,
	title= {{CIFAR-10} (Canadian Institute for Advanced Research)},
	journal= {},
	author= {Alex Krizhevsky and Vinod Nair and Geoffrey Hinton},
	year= {2009},
	url= {http://www.cs.toronto.edu/~kriz/cifar.html},
}

@inproceedings{deng2009imagenet,
  title={Imagenet: A large-scale hierarchical image database},
  author={Deng, Jia and Dong, Wei and Socher, Richard and Li, Li-Jia and Li, Kai and Fei-Fei, Li},
  booktitle={CVPR},
  year={2009},
}

@article{wang2022guided,
  title={Guided diffusion model for adversarial purification},
  author={Wang, Jinyi and Lyu, Zhaoyang and Lin, Dahua and Dai, Bo and Fu, Hongfei},
  journal={arXiv preprint arXiv:2205.14969},
  year={2022}
}

@inproceedings{zhang2023detecting,
  author       = {Shuhai Zhang and
                  Feng Liu and
                  Jiahao Yang and
                  Yifan Yang and
                  Changsheng Li and
                  Bo Han and
                  Mingkui Tan},
  title        = {Detecting Adversarial Data by Probing Multiple Perturbations Using
                  Expected Perturbation Score},
  booktitle    = {ICML},
  year         = {2023},
}

@inproceedings{gao2021maximum,
  author       = {Ruize Gao and
                  Feng Liu and
                  Jingfeng Zhang and
                  Bo Han and
                  Tongliang Liu and
                  Gang Niu and
                  Masashi Sugiyama},
  title        = {Maximum Mean Discrepancy Test is Aware of Adversarial Attacks},
  booktitle    = {ICML},
  year         = {2021},
}

@inproceedings{tramer2020on,
  author       = {Florian Tram{\`{e}}r and
                  Nicholas Carlini and
                  Wieland Brendel and
                  Aleksander Madry},
  title        = {On Adaptive Attacks to Adversarial Example Defenses},
  booktitle    = {NeurIPS},
  year         = {2020},
}

@inproceedings{hill2021stochastic,
  author       = {Mitch Hill and
                  Jonathan Craig Mitchell and
                  Song{-}Chun Zhu},
  title        = {Stochastic Security: Adversarial Defense Using Long-Run Dynamics of
                  Energy-Based Models},
  booktitle    = {ICLR},
  year         = {2021},
}

@inproceedings{zhuang2020adaptive,
  author       = {Juntang Zhuang and
                  Nicha C. Dvornek and
                  Xiaoxiao Li and
                  Sekhar Tatikonda and
                  Xenophon Papademetris and
                  James S. Duncan},
  title        = {Adaptive Checkpoint Adjoint Method for Gradient Estimation in Neural
                  {ODE}},
  booktitle    = {ICML},
  year         = {2020},
}

@inproceedings{song2021denoising,
  author       = {Jiaming Song and
                  Chenlin Meng and
                  Stefano Ermon},
  title        = {Denoising Diffusion Implicit Models},
  booktitle    = {ICLR},
  year         = {2021},
}

@inproceedings{li2020scalable,
  author       = {Xuechen Li and
                  Ting{-}Kam Leonard Wong and
                  Ricky T. Q. Chen and
                  David Duvenaud},
  title        = {Scalable Gradients for Stochastic Differential Equations},
  booktitle    = {The 23rd International Conference on Artificial Intelligence and Statistics},
  pages        = {3870--3882},
  year         = {2020},
}

@inproceedings{chen2024robust,
  author       = {Huanran Chen and
                  Yinpeng Dong and
                  Zhengyi Wang and
                  Xiao Yang and
                  Chengqi Duan and
                  Hang Su and
                  Jun Zhu},
  title        = {Robust Classification via a Single Diffusion Model},
  booktitle    = {ICML},
  year         = {2024},
}

@inproceedings{croce2020reliable,
  author       = {Francesco Croce and
                  Matthias Hein},
  title        = {Reliable evaluation of adversarial robustness with an ensemble of
                  diverse parameter-free attacks},
  booktitle    = {ICML},
  year         = {2020},
}

@inproceedings{athalye2018obfuscated,
  author       = {Anish Athalye and
                  Nicholas Carlini and
                  David A. Wagner},
  title        = {Obfuscated Gradients Give a False Sense of Security: Circumventing
                  Defenses to Adversarial Examples},
  booktitle    = {ICML},
  year         = {2018},
}

@inproceedings{kidger2021neuralsde,
  title={Neural {SDE}s as {I}nfinite-{D}imensional {GAN}s},
  author={Kidger, Patrick and Foster, James and Li, Xuechen and Oberhauser, Harald and Lyons, Terry},
  booktitle={ICML},
  year={2021}
}

@inproceedings{dhariwal2021diffusion,
  author       = {Prafulla Dhariwal and
                  Alexander Quinn Nichol},
  title        = {Diffusion Models Beat GANs on Image Synthesis},
  booktitle    = {NeurIPS},
  year         = {2021},
}

@inproceedings{bai2024diffusion,
  author       = {Mingyuan Bai and
                  Wei Huang and
                  Tenghui Li and
                  Andong Wang and
                  Junbin Gao and
                  Cesar F. Caiafa and
                  Qibin Zhao},
  title        = {Diffusion Models Demand Contrastive Guidance for Adversarial Purification to Advance},
  booktitle    = {ICML},
  year         = {2024},
}

@article{chen2023diffusion,
      title={Diffusion Models for Imperceptible and Transferable Adversarial Attack}, 
      author={Jianqi Chen and Hao Chen and Keyan Chen and Yilan Zhang and Zhengxia Zou and Zhenwei Shi},
    journal = {IEEE Transactions on Pattern Analysis and Machine Intelligence},
      year={2024},
}

@inproceedings{xue2024diffusion,
      title={Diffusion-Based Adversarial Sample Generation for Improved Stealthiness and Controllability}, 
      author={Haotian Xue and Alexandre Araujo and Bin Hu and Yongxin Chen},
  booktitle    = {NeurIPS},
      year={2023},
}

@inproceedings{liu2021swin,
  author       = {Ze Liu and
                  Yutong Lin and
                  Yue Cao and
                  Han Hu and
                  Yixuan Wei and
                  Zheng Zhang and
                  Stephen Lin and
                  Baining Guo},
  title        = {Swin Transformer: Hierarchical Vision Transformer using Shifted Windows},
  booktitle    = {ICCV},
  year         = {2021},
}

@inproceedings{anish2018synthesizing,
  author       = {Anish Athalye and
                  Logan Engstrom and
                  Andrew Ilyas and
                  Kevin Kwok},
  title        = {Synthesizing Robust Adversarial Examples},
  booktitle    = {ICML},
  year         = {2018},
}

@inproceedings{zhang2024improving,
  author       = {Jiacheng Zhang and
                  Feng Liu and
                  Dawei Zhou and
                  Jingfeng Zhang and
                  Tongliang Liu},
  title        = {Improving Accuracy-robustness Trade-off via Pixel Reweighted Adversarial Training},
  booktitle    = {ICML},
  year         = {2024},
}

@misc{zhang2025ddad,
author       = {Jiacheng Zhang and
                  Benjamin I. P. Rubinstein and
                  Jingfeng Zhang and
                  Feng Liu},
title        = {One Stone, Two Birds: Enhancing Adversarial Defense Through the Lens of Distributional Discrepancy},
eprint={2503.02169},
archivePrefix={arXiv},
url={https://arxiv.org/abs/2503.02169}, 
}

@inproceedings{zhang2021geometry,
  author       = {Jingfeng Zhang and
                  Jianing Zhu and
                  Gang Niu and
                  Bo Han and
                  Masashi Sugiyama and
                  Mohan S. Kankanhalli},
  title        = {Geometry-aware Instance-reweighted Adversarial Training},
  booktitle    = {ICLR},
  year         = {2021},
}

@inproceedings{wang2021probabilistic,
  author       = {Qizhou Wang and
                  Feng Liu and
                  Bo Han and
                  Tongliang Liu and
                  Chen Gong and
                  Gang Niu and
                  Mingyuan Zhou and
                  Masashi Sugiyama},
  title        = {Probabilistic Margins for Instance Reweighting in Adversarial Training},
  booktitle    = {NeurIPS},
  year         = {2021},
}

@inproceedings{ma2018characterizing,
  author       = {Xingjun Ma and
                  Bo Li and
                  Yisen Wang and
                  Sarah M. Erfani and
                  Sudanthi N. R. Wijewickrema and
                  Grant Schoenebeck and
                  Dawn Song and
                  Michael E. Houle and
                  James Bailey},
  title        = {Characterizing Adversarial Subspaces Using Local Intrinsic Dimensionality},
  booktitle    = {ICLR},
  year         = {2018},
}

@inproceedings{lee2018a,
  author       = {Kimin Lee and
                  Kibok Lee and
                  Honglak Lee and
                  Jinwoo Shin},
  title        = {A Simple Unified Framework for Detecting Out-of-Distribution Samples
                  and Adversarial Attacks},
  booktitle    = {NeurIPS},
  year         = {2018},
}

@inproceedings{raghuram2021a,
  author       = {Jayaram Raghuram and
                  Varun Chandrasekaran and
                  Somesh Jha and
                  Suman Banerjee},
  title        = {A General Framework For Detecting Anomalous Inputs to {DNN} Classifiers},
  booktitle    = {ICML},
  year         = {2021},
}

@inproceedings{pang2022two,
  author       = {Tianyu Pang and
                  Huishuai Zhang and
                  Di He and
                  Yinpeng Dong and
                  Hang Su and
                  Wei Chen and
                  Jun Zhu and
                  Tie{-}Yan Liu},
  title        = {Two Coupled Rejection Metrics Can Tell Adversarial Examples Apart},
  booktitle    = {CVPR},
  year         = {2022},
}

@inproceedings{liao2018defense,
  author       = {Fangzhou Liao and
                  Ming Liang and
                  Yinpeng Dong and
                  Tianyu Pang and
                  Xiaolin Hu and
                  Jun Zhu},
  title        = {Defense Against Adversarial Attacks Using High-Level Representation
                  Guided Denoiser},
  booktitle    = {CVPR},
  year         = {2018},
}

@inproceedings{samangouei2018defensegan,
  author       = {Pouya Samangouei and
                  Maya Kabkab and
                  Rama Chellappa},
  title        = {{Defense-GAN}: Protecting Classifiers Against Adversarial Attacks Using
                  Generative Models},
  booktitle    = {ICLR},
  year         = {2018},
}

@inproceedings{song2018pixeldefend,
  author       = {Yang Song and
                  Taesup Kim and
                  Sebastian Nowozin and
                  Stefano Ermon and
                  Nate Kushman},
  title        = {{PixelDefend}: Leveraging Generative Models to Understand and Defend
                  against Adversarial Examples},
  booktitle    = {ICLR},
  year         = {2018},
}

@inproceedings{naseer2020a,
  author       = {Muzammal Naseer and
                  Salman H. Khan and
                  Munawar Hayat and
                  Fahad Shahbaz Khan and
                  Fatih Porikli},
  title        = {A Self-supervised Approach for Adversarial Robustness},
  booktitle    = {CVPR},
  year         = {2020},
}

@book{bo2025trustworthy,
  author={Han, Bo and Yao, Jiangchao and Liu, Tongliang and Li, Bo and Koyejo, Sanmi and Liu, Feng},
  title={Trustworthy Machine Learning: From Data to Models},
  year={2025},
  publisher={Now Foundations and Trends}
}
